\documentclass[11pt,letterpaper]{article}
\usepackage[dvipsnames]{xcolor}
\usepackage{tikz}
\usetikzlibrary{angles,quotes,calc}
\usetikzlibrary{
  shapes.multipart,
  matrix,
  positioning,
  shapes.callouts,
  shapes.arrows,
  calc,
  decorations.pathreplacing}

\usepackage{graphicx}
\usepackage[shortlabels]{enumitem}
\usepackage{multicol}
\usepackage{mathrsfs}
\usepackage[margin=1in]{geometry} 
\usepackage{amsmath,amsthm,amssymb}
\usepackage{mathtools}
\usepackage[utf8]{inputenc}

\usepackage[numbers,square,sort&compress]{natbib}

\usepackage{bbm}
\usepackage{bm}
\usepackage{hyperref}
\usepackage{multirow}
\usepackage[capitalise, nameinlink, noabbrev]{cleveref}
\usepackage{xifthen}
\usepackage{booktabs}
\hypersetup{
    colorlinks=true,
    linkcolor=blue,
    filecolor=magenta,      
    urlcolor=cyan,
    citecolor=Green
}
\usepackage{thmtools}
\usepackage{thm-restate}

\newtheorem{theorem}{Theorem}[section]

\newtheorem{corollary}[theorem]{Corollary}
\newtheorem{lemma}[theorem]{Lemma}

\newtheorem{definition}[theorem]{Definition}
\newtheorem{remark}[theorem]{Remark}
\newtheorem{assumption}{Assumption}
\newtheorem{proposition}[theorem]{Proposition}

\newcommand{\R}{\mathbb{R}}

\DeclareMathOperator*{\E}{\mathbb{E}}

\DeclareMathOperator*{\poly}{poly}

\newcommand{\eps}{\varepsilon}

\newcommand{\wt}{\widetilde}

\newcommand{\Prb}[2][]{ \ifthenelse{\isempty{#1}}
  {\Pr\left[#2\right]}
  {\Pr_{#1}\left[#2\right]} }
\newcommand{\Ex}[2][]{ \ifthenelse{\isempty{#1}}
  {\E\left[#2\right]}
  {\E_{#1}\left[#2\right]} }
\newcommand{\Vari}[2][]{ \ifthenelse{\isempty{#1}}
  {\mathbf{Var}\left[#2\right]}
  {\mathop{\mathbf{Var}}_{#1}\left[#2\right]} }
\newcommand{\nPrb}[2][]{ \ifthenelse{\isempty{#1}}
  {\Pr[#2]}
  {\Pr_{#1}[#2]} }
\newcommand{\nEx}[2][]{ \ifthenelse{\isempty{#1}}
  {\E[#2]}
  {\E_{#1}[#2]} }

\newcommand{\Hd}{\{-1,+1\}^d}

\usepackage{subcaption}

\DeclareMathOperator{\Unif}{Unif}

\DeclareMathOperator{\TV}{d_{\mathrm{TV}}}
\DeclareMathOperator{\KL}{D_{\mathrm{KL}}}
\DeclareMathOperator{\Lip}{Lip}

\newcommand{\Prob}{\mathbb P}
\newcommand{\cN}{\mathcal N}

\newcommand{\err}{\mathrm{err}}

\author{Zhiyang Xun \\ UT Austin \\ \texttt{zxun@cs.utexas.edu} \and Eric Price \\ UT Austin \\ \texttt{ecprice@cs.utexas.edu}}

\title{Query Lower Bounds for Diffusion Sampling}
\usepackage{pgfplots}
\pgfplotsset{compat=1.18}
\usetikzlibrary{arrows.meta,calc,positioning}

\begin{document}

\maketitle
\begin{abstract}
  Diffusion models generate samples by iteratively querying learned score estimates. A rapidly growing literature focuses on accelerating sampling by minimizing the number of score evaluations, yet the information-theoretic limits of such acceleration remain unclear.

  In this work, we establish the first score query lower bounds for diffusion sampling. We prove that for $d$-dimensional distributions, given access to score estimates with polynomial accuracy $\varepsilon=d^{-O(1)}$ (in any $L^p$ sense), any sampling algorithm requires $\widetilde{\Omega}(\sqrt{d})$ adaptive score queries. In particular, our proof shows that, within any polynomial total-query budget, successful sampling requires searching over $\widetilde{\Omega}(\sqrt{d})$ distinct noise levels, providing a formal explanation for why multiscale noise schedules are necessary in practice.
\end{abstract}

\section{Introduction}
Diffusion models have become a central paradigm in modern generative modeling, enabling major advances in tasks ranging from high-fidelity image synthesis to scientific computing \citep{sohl2015diffusion,ho2020ddpm,song2021score_sde,nichol2021improved,karras2022edm}.
The key to their success lies in an iterative formulation: rather than generating samples directly in a single shot, these models progressively transform simple noise into structured data by repeatedly evaluating the \textit{smoothed scores} of the distribution \citep{song2019sgm,song2021score_sde}.
By reducing the task of sampling from complex high-dimensional multimodal distributions to estimating the scores, this framework has dramatically improved generative modeling capabilities.

A central goal within this framework is improving sampling efficiency.
Since each iteration typically requires evaluating a large neural network, the computational cost is dominated by the total number of such evaluations (queries).
This has motivated a large body of work on reducing query complexity, including higher-order numerical solvers, optimized discretization schemes, and flow matching techniques that straighten generation trajectories  \citep{lu2022dpmsolver,lu2022dpmsolverpp,liu2022pndm,zhang2022deis,zhao2023unipc,karras2022edm,lipman2022flowmatching,liu2022rectified}.
These advances have steadily reduced the number of queries required in both theory and practice.

This rapid progress raises a basic theoretical question: as we aim to reduce the number of score queries, what are the intrinsic information-theoretic limits of diffusion sampling in high dimension?
When access to the target distribution is limited to querying smoothed score estimates, is there an unavoidable lower bound on the number of such queries?
In this work, we ask:
\begin{quote}
  \centering
  \textit{Can we establish a query-complexity lower bound for score-based sampling?}
\end{quote}

Recent theory has made substantial progress on \emph{upper} bounds~\citep{chen2022samplingiseasy,lee2022polynomialcomplexity}.
In particular, for targets with bounded second moments, one can sample in $\widetilde O(d)$ iterations given access to $L^2$-accurate smoothed score estimates~\citep{benton2023linearconvergence}.
At the same time, practical diffusion samplers often produce high-quality
samples in only a handful of steps, far fewer than worst-case theoretical
guarantees.
This gap makes lower bounds essential: is a polynomial dependence on~$d$
truly unavoidable, or could the query complexity be reduced to
$\mathrm{polylog}(d)$, or even to $O(1)$?

 Ruling out this possibility turns out to be very nontrivial. Information-theoretically, \textit{smoothed} score estimates are very powerful in sampling, suggesting that such efficiency improvements might indeed be feasible:
 \begin{enumerate}
    \item Each query of a smoothed score returns a full vector in $\mathbb{R}^d$, providing $\Omega(d)$ bits of information.
Thus, if each query can extract a near-optimal amount of information, the intrinsic difficulty of sampling might not necessarily grow with dimension.
\item Although scores are by definition local quantities, the smoothed scores can reveal global properties: Gaussian convolution can encode global distributional structure into local gradients, potentially enabling efficient navigation of high-dimensional space.

\item  Even though $L^2$-accurate score estimates only guarantee accuracy in expectation rather than at any individual query point, large smoothing levels can mitigate this gap: they ensure that queries land within the typical set of the smoothed distribution, where the $L^2$ guarantee translates into high-probability pointwise accuracy.
\end{enumerate}
Furthermore, recent results hint at a gap between limitations of specific discretizations and the potential power of smoothed-score access.
For instance, \cite{jiao2025optimalddpm,gao2025sdewasserstein} show that even for a $d$-dimensional Gaussian, standard discretizations can require $\Omega(\sqrt d)$ steps due to error accumulation.
On the other hand, work on high-accuracy regimes has begun to explore $\varepsilon$-accurate sampling with $\mathrm{polylog}(1/\varepsilon)$ dependence~\citep{gatmiry2026highaccuracy,chen2026highaccuracy}, raising the possibility of sampling with $\mathrm{polylog}(d)$ steps on product distributions.
These considerations leave open whether there is a fundamental, information-theoretic barrier to few-step diffusion sampling in the worst case.

\subsection{Our Results}
In this work, we establish the first query-complexity lower bounds for diffusion sampling. Our main message is that, under standard distributional assumptions and polynomially accurate score estimates, even achieving TV error below $0.99$ requires $\widetilde\Omega(\sqrt d)$ score queries in the worst case.

To state our bounds formally, we first introduce the necessary notation. Let $\pi$ be  a target distribution on $\mathbb{R}^d$. For any noise level $\sigma>0$ we write
\[
\pi_\sigma = \pi * \mathcal{N}(0,\sigma^2 I_d),
\qquad
s_{\pi,\sigma}(x)=\nabla \log \pi_\sigma(x),
\]
for the Gaussian-smoothed distribution and its score. We model diffusion samplers as follows:

\begin{definition}[Diffusion Sampling]
A diffusion sampling algorithm $\mathcal{A}$ accesses a target distribution $\pi$ {exclusively} via adaptive queries to an oracle for smoothed score estimates $\widehat s_\sigma(x)$, outputting a sample $\widehat X$ whose law aims to approximate $\pi$.
\end{definition}

To guarantee convergence, assumptions on the class of target distributions and the quality of the score estimates are needed. We work under two \textit{strong} conditions (i.e., favorable to the algorithm) that are standard in the diffusion-theory literature; therefore, our lower bounds apply to a broad range of algorithms.

First, we assume the target is a Gaussian smoothing of a bounded-support distribution.

\begin{assumption}[Bounded plus noise]
  \label{ass:bounded}
  For constant $R > 0$ and constant $\gamma \in (0, R / 2)$, there exists a distribution $\pi_{\mathrm{pre}}$ supported on
  $[-R, R]^d$ such that \[
    \pi = \pi_{\mathrm{pre}} * \mathcal{N}(0, \gamma^2 I_d).
  \]
\end{assumption}

Bounded-plus-noise is stronger than bounded moments or subgaussian tails, and it automatically implies $M_2=\E[\|X\|_2^2]=O(d)$ as well as bounds on all higher moments.
The additional Gaussian smoothing ensures that $\pi$ is infinitely differentiable and mirrors the smoothed target distributions typically sampled by diffusion models (e.g., via early stopping). Furthermore, this smoothing makes Total Variation (TV) convergence guarantees possible for sampling algorithms.

Second, we assume the score estimate is polynomially $L^p$-accurate. The case $p = 2$ matches the standard assumption in diffusion theory analysis, as it aligns with the standard score matching training objective used in practice~\citep{hyvarinen2005score,vincent2011connection,song2019sgm,ho2020ddpm,song2021score_sde,improvedSample}. We state the assumption for any constant $p \ge 2$ to cover a wider range of algorithms~\citep{L41, xun2025posterior}.

\begin{assumption}[$L^p$-accurate score]
  \label{ass:Lp}
    There exists a constant $p \ge 2$ and $\varepsilon_\err = 1 / \poly(d)$ such that for all $\sigma > 0$,
  \[
    \E_{X\sim \pi_\sigma}\big[\|\widehat s_\sigma(X)-s_{\pi,\sigma}(X)\|_2^p\big]\ \le\ \frac{\varepsilon_\err^p}{\sigma^p}.
  \]
\end{assumption}

Under Assumptions~\ref{ass:bounded} and~\ref{ass:Lp}, the best known upper bound is that $\widetilde O(d)$ queries suffice to achieve small TV error. The bound can be achieved by analyzing the discretization error of standard algorithms like DDPM~\citep{benton2023linearconvergence,li2024odt,conforti2023score}. We restate the result of~\cite{benton2023linearconvergence} below, translating their original KL bound to TV distance to facilitate comparison:

\begin{theorem}[\cite{benton2023linearconvergence}]
    Under Assumptions~\ref{ass:bounded} and~\ref{ass:Lp}, there exists a diffusion sampling algorithm (DDPM) that achieves TV error $0.01$ using $\widetilde O(d)$ queries.
\end{theorem}

Our main theorem shows that one cannot reduce the query complexity below  $\widetilde \Omega(\sqrt{d})$ in the worst case, even when allowing TV error as large as 0.99; see~\cref{fig:query-complexity-line} for a comparison with existing upper bounds.

\begin{theorem}[Main Theorem]
  \label{thm:main_simplified}
   Any diffusion sampling algorithm under Assumptions~\ref{ass:bounded} and~\ref{ass:Lp} with TV error less than $0.99$ requires at least $\widetilde\Omega(\sqrt{d})$ queries.
\end{theorem}

\begin{figure}[t]
\centering
\begin{tikzpicture}[
    x=1cm,
    >=stealth,
    every node/.style={font=\small},
    tick/.style={line width=0.8pt},
    dot/.style={circle, fill=blue!70!black, inner sep=2.8pt},
    mainlabel/.style={align=center, font=\small},
    refs/.style={
        anchor=north,
        align=center,
        font=\small,            
        text=blue!70!black,
        text width=2.9cm
    },
    theoremrefs/.style={
        anchor=north,
        align=center,
        font=\small,            
        text=red!80!black,
        text width=2.4cm
    }
]

\draw[<->, line width=0.9pt] (0.3,0) -- (14.3,0);

\node[font=\small\itshape, text=gray] at (1.4,0.95) {fewer queries};
\node[font=\small\itshape, text=gray] at (13.25,0.95) {more queries};

\draw[tick] (2.5,0.12) -- (2.5,-0.12);
\node[mainlabel] at (2.5,-0.56) {$\operatorname{polylog}(d)$};

\node[font=\LARGE, text=gray] at (2.5,0.5) {$\times$};

\draw[red!80!black, line width=1.6pt] (5.0,-0.18) -- (5.0,0.18);
\node[mainlabel, text=red!80!black] at (5.0,-0.56)
    {$\widetilde{\Omega}(\sqrt d)$};
\node[theoremrefs] at (5.0,-0.95)
    {\Cref{thm:main_simplified}};

\node[dot] at (7.55,0) {};
\node[mainlabel] at (7.55,-0.56) {$\widetilde O(d)$};

\node[refs] at (7.55,-0.95)
    {\citealp{benton2023linearconvergence,li2024odt,conforti2023score,jiao2025optimalddpm,chen2026highaccuracy}};

\node[dot] at (9.9,0) {};
\node[mainlabel] at (9.9,-0.56) {$\widetilde O(d^2)$};
\node[refs] at (9.9,-0.95)
    {\citealp{chen2023improved}};

\node[dot] at (12.25,0) {};
\node[mainlabel] at (12.25,-0.56) {$\operatorname{poly}(d)$};
\node[refs] at (12.25,-0.95)
    {\citealp{chen2022samplingiseasy,lee2023convergence}};

\end{tikzpicture}
\caption{Score query complexity of diffusion sampling under Assumptions \ref{ass:bounded} and \ref{ass:Lp}: representative upper bounds from prior work, alongside our $\widetilde\Omega(\sqrt d)$ lower bound, which rules out $\operatorname{polylog}(d)$ query complexity in this setting.}

\label{fig:query-complexity-line}
\end{figure}

More recently, researchers have shown that beyond Assumptions~\ref{ass:bounded} and~\ref{ass:Lp}, assuming additional Lipschitz regularity of the score can lead to algorithm acceleration, achieving sublinear iteration complexity in $d$~\citep{gupta2024randomized,li2025improved}. For example,~\cite{zhang2025sublinear} showed that if the score estimates are $L$-Lipschitz, then DDPM can require only $\widetilde O(L\sqrt{d})$ score queries; under their score-accuracy assumption, \cite{jiaoli2024instancedependent} gave an algorithm with~$\widetilde O(\min\{d,L^{1/3}d^{2/3},Ld^{1/3}\})$ queries, where $L$ is the Lipschitz constant of the true scores.

Therefore, a useful perspective is that our hard instance has globally Lipschitz true smoothed scores and score estimates, with a Lipschitz parameter that scales linearly with dimension.

\begin{remark}[Lipschitz scale of the smoothed score]
\label{rmk:lip}
In the hard instance underlying \cref{thm:main_simplified}, for every $\sigma>0$, both $s_{\pi,\sigma}$ and $\widehat s_\sigma$ are globally $O(d)$-Lipschitz.
\end{remark}

This implies that any algorithm with a query upper bound of $\widetilde O(L^a d^b)$ yields a complexity of $\widetilde O(d^{a+b})$ when instantiated on our hard instance.
Theorem~\ref{thm:main_simplified} therefore rules out query upper bounds of the form $\widetilde O(L^a d^b)$ with $a+b<1/2$.

A more formal version of the theorem, which tracks the dependence on $(R,\gamma)$, the oracle accuracy, and the explicit Lipschitz bound, is given in \cref{thm:main_parameterized}.

At the proof level, our lower bound isolates one intrinsic source
of hardness: within any polynomial total-query budget,
a successful diffusion sampler must scan through
$\widetilde\Omega(\sqrt{d})$ distinct noise levels.
This matches what diffusion algorithms do in practice, where scores
are queried along a multiscale noise schedule, and provides a
formal explanation for why such a design is necessary.
Our proof proceeds by reduction to a hypothesis-testing problem:
we construct a null distribution $\pi_{\mathrm{null}}$ and a class
of planted distributions~$\mathcal{D}$, and show that
$\widetilde\Omega(\sqrt{d})$ queries at distinct noise levels are
needed to distinguish whether the score estimates come from
$\pi_{\mathrm{null}}$ or from some $\pi$ drawn
from~$\mathcal{D}$.
In fact, with fewer queries the algorithm's execution on a random planted
instance is, with high probability, identical to its execution on the null
instance, so its output law is nearly a fixed distribution, which cannot
approximate a random planted target.
Section~\ref{sec:roadmap} outlines the proof.

Under Assumptions~\ref{ass:bounded} and~\ref{ass:Lp}, a gap
remains between our $\widetilde\Omega(\sqrt{d})$ lower bound and
the best known $\widetilde O(d)$ upper bound.
Closing this gap is an important open problem, and progress could
come from either direction: stronger lower bounds that go beyond
hypothesis testing to directly exploit the difficulty of producing
a sample, or faster algorithms that leverage the structure of the
noise-level scanning bottleneck identified here.

\paragraph{Wasserstein Distance.}
Another line of work analyzes diffusion samplers in Wasserstein distance.
\cite{arsenyan2025wasserstein} show that $\widetilde O(\sqrt d/\varepsilon)$
queries to a uniformly accurate score suffice for $W_2$ error $\varepsilon$
under log-concavity-type conditions, and \cite{gentilonisilveri2025beyond}
obtain $W_2$ guarantees under weak log-concavity; both bounds depend
exponentially on parameters quantifying non-log-concavity.
The same hard instances give a Wasserstein analog of \cref{thm:main_simplified}.

\begin{theorem}[Wasserstein Lower Bound]
  \label{thm:wasserstein_simplified}
  There is a constant $c>0$ such that, for every $q\ge1$, any
  diffusion sampling algorithm under Assumptions~\ref{ass:bounded}
  and~\ref{ass:Lp} with $W_q$ error less than $c\sqrt d$ requires at least
  $\widetilde\Omega(\sqrt d)$ queries.
\end{theorem}

Outputting the origin, with no queries, already achieves $W_q$ error
$O(R\sqrt d)$ for every such $\pi$ and fixed $q$, so
\cref{thm:wasserstein_simplified} says that, with fewer queries, no
algorithm improves on this trivial output by more than a constant factor.

\paragraph{Subexponential  Error Tail.}
While $L^p$ accuracy is the standard assumption, it is also instructive to consider the impact of stronger error tail conditions. Recent work has shown that stronger error tail requirements might yield accelerated algorithms~\citep{gatmiry2026highaccuracy}.

\begin{assumption}[Subexponential Error]
\label{ass:psi1}
For all $\sigma > 0$ and all $z\ge 0$, the model can query $\widehat s_\sigma(x)$ with the guarantee that
\[
\Pr_{X\sim \pi_\sigma}\Big[\|\widehat s_\sigma(X)-s_{\pi,\sigma}(X)\|_2\ge z\Big]
\ \le\
2\exp\left(-\frac{z\sigma}{\varepsilon_{\rm err}}\right).
\]
\end{assumption}

We show that, under a constant subexponential error model, $\Omega(d^{1/4})$ queries are still necessary.

\begin{theorem}[Subexponential Tail]
  \label{thm:psi1_simplified}
  For any constant $\varepsilon_{\mathrm{err}} > 0$, any algorithm that solves diffusion sampling under Assumptions~\ref{ass:bounded} and~\ref{ass:psi1} with TV error less than $0.99$ requires at least $\Omega(d^{1/4})$ queries.
\end{theorem}

\section{Proof Overview}
\label{sec:overview}

\subsection{Intuition on Hardness}
\label{sec:intuition}
To build intuition, we consider a simple hard family: mixtures of $n$ well-separated Gaussians.
Concretely, pick a uniformly random set of centers $S\subseteq\Hd$ with $|S|=n$, and define
\[
\pi_{\rm pre}=\Unif(S), \qquad \pi=\pi_{\rm pre}*\mathcal N(0,0.1 I_d).
\]
Clearly, any nontrivial sampler must locate at least one center in $S$ using score queries.
Our proof shows that this is only possible when the algorithm queries a \emph{narrow window} of smoothing levels determined by $n$.
Since $n$ is hidden and may range over $2^{\Theta(d)}$ possibilities, the sampler is forced to scan many different noise levels.
 
{Throughout this subsection, we write $\tau$ for the total noise
level, setting $\pi_\tau := \pi_{\rm pre} * \mathcal N(0, \tau^2
I_d)$ with score $s_{\pi,\tau}$, and write $s^{\rm null}_\tau$
for the null score obtained by replacing $S$ with $\Hd$.}
 
By Tweedie's formula, for $X=Y+Z$ with $Y\sim\pi_{\rm pre}$, $Z\sim\mathcal{N}(0,\tau^2 I_d)$, so that $X\sim\pi_\tau$, the score can be rewritten as
\[
s_{\pi, {\tau}}(x)=\frac{m(x)-x}{ {\tau}^2},
\]
where $m(x) = \E[Y\mid X=x]$.
Since $Y$ is supported on $S$, $m(x)$ is a posterior average over the modes $y\in S$, with weights proportional to
$\exp(-\|x-y\|_2^2/(2 {\tau}^2))$.
Thus a score query is informative only insofar as this posterior average \emph{depends on the planted set $S$}.

As a reference,  {recall that} $s^{\rm null}_{ {\tau}}$  {is} the score one would obtain if the adversary pretended $S=\Hd$ (so the answer carries no information about $S$).
We will see that for most $ {\tau}$, an adversary can answer essentially with $s^{\rm null}_{ {\tau}}$ while still satisfying the polynomial $L^p(\pi_{ {\tau}})$ accuracy guarantee, forcing the sampler to scan across~$ {\tau}$.

\paragraph{Warmup: Gaussian as a sphere.}
A Gaussian in $\R^d$ concentrates on a thin shell: $\|Z\|_2\approx  {\tau}\sqrt d$.
Imagine that it were \emph{exactly uniform} on the sphere of radius $ {\tau}\sqrt d$.
Then for fixed $(x, {\tau})$, $m(x)$ would simply average the sampled centers lying in the corresponding sphere around $x$.
Let $N_{x, \tau}$ denote the expected number of points of $S$ falling in this sphere.
 
This yields two uninformative extremes:
\begin{itemize}
\item 
If $N_{x, \tau}\gg \poly(d)$, with high probability the averaging washes out the instance: by concentration, $m(x)$ becomes $1/\poly(d)$-close to the null answer, so an adversary can answer with $s^{\rm null}_{ {\tau}}$.
\item 
If $N_{x, \tau}\ll 1/\poly(d)$, then with high probability over $S$ there are no points in the sphere. This implies that $\pi_{ {\tau}}(x)$ has tiny mass, and the $L^p(\pi_{ {\tau}})$ guarantee does not constrain the oracle on these points. The adversary can again answer with $s^{\rm null}_{ {\tau}}$.
\end{itemize}

\begin{figure}[t]
  \centering
  \begin{tikzpicture}
    \node[anchor=south west, inner sep=0] (img) at (0,0)
      {\includegraphics[width=\linewidth]{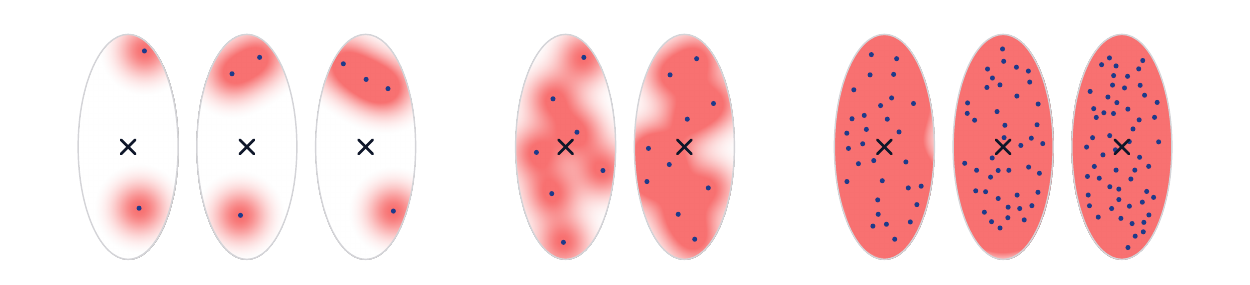}};
    \begin{scope}[x={(img.south east)}, y={(img.north west)}]
      \node[font=\Large, color=black!55] at (0.3725, 0.4835) {$\cdots$};
      \node[font=\Large, color=black!55] at (0.6275, 0.4835) {$\cdots$};
      \draw[decorate, decoration={brace, amplitude=5pt, mirror},
            line width=0.7pt, color=black!55]
        (0.0625, -0.05) -- (0.3325, -0.05)
        node[midway, below=4pt, font=\itshape, color=black!60] {off-distribution};
      \draw[decorate, decoration={brace, amplitude=5pt, mirror},
            line width=0.7pt, color=orange!70!black]
        (0.4125, -0.05) -- (0.5875, -0.05)
        node[midway, below=4pt, font=\itshape, color=orange!70!black] {informative};
      \draw[decorate, decoration={brace, amplitude=5pt, mirror},
            line width=0.7pt, color=black!55]
        (0.6675, -0.05) -- (0.9375, -0.05)
        node[midway, below=4pt, font=\itshape, color=black!60] {averaging out};
    \end{scope}
  \end{tikzpicture}
 \caption{\textit{Three regimes of the smoothed-score query at fixed $\tau$.} With high probability, only a narrow range of $|S|=n$ produces an informative query at $\times$.}
  \label{fig:hardness-regimes}
\end{figure}

Either way, a query at this $ {\tau}$ reveals essentially nothing except that the chosen smoothing level is ``wrong.'' The mechanism is illustrated in~\cref{fig:hardness-regimes}.

To get more information, $ {\tau}$ must be tuned so that $N_{x, \tau}$  lies within an intermediate window of a $\poly(d)$ range. 
Note that $N_{x, \tau}$ scales linearly with the underlying $n$.
Therefore, each query would only rule out the range of $n$ up to a $\poly(d)$ factor, and there are \[
\log_{\poly(d)} (2^{\Theta(d)}) = \Theta(d/\log d)
\]
such ranges to check. Therefore, we need to query $\Theta(d/\log d)$ distinct smoothing levels to get one ``not null'' answer that actually helps sampling.

\paragraph{The real Gaussian gives $\sqrt d$.}
 
The analysis above fails in one place: a Gaussian is \emph{not} uniform even on its thin shell.
With $1 - 1/\poly(d)$ probability,
\[
\|Z\|_2^2= {\tau}^2 d \pm \wt{O}( {\tau}^2\sqrt d),
\]
so the posterior is influenced by a whole band of squared distances of width $\Theta( {\tau}^2\sqrt d)$, and the mass outside is negligible.

Pick two squared radii $r_\pm^2= {\tau}^2 d \pm \Theta( {\tau}^2\sqrt d)$ within this typical band.
The per-center likelihood weights satisfy
\[
\frac{\exp(-r_-^2/(2 {\tau}^2))}{\exp(-r_+^2/(2 {\tau}^2))}
=\exp\!\left(\Theta\!\left(\frac{r_+^2-r_-^2}{ {\tau}^2}\right)\right)
=\exp(\Theta(\sqrt d)).
\]
 
At the same time, the size of the corresponding Hamming shell at radius $r_+$ versus $r_-$ differs by
$\exp(\Theta(\sqrt d))$ as well.
 
When $S$ is a random subset of size $n$, this compensation creates a gap between the two thresholds in the counting story:
as $n$ grows, one first exits the ``empty'' regime by hitting the high-surface-area part of the band, but only after an additional
$\exp(\wt\Theta(\sqrt d))$ factor in $n$ does one reach the true ``averaging'' regime, where even the low-surface-area (but high-weight) part is well populated.
Equivalently, for fixed $ {\tau}$ the informative window can pin down $n$ only up to an $\exp(\wt\Theta(\sqrt d))$ multiplicative factor.

Therefore we must scan
\[
\log_{\exp(\wt\Theta(\sqrt d))}\!\big(\exp(\Theta(d))\big)=\wt\Theta(\sqrt d)
\]
distinct smoothing levels to reliably hit an informative scale, yielding the $\wt{\Omega}(\sqrt d)$ lower bound.
 
\begin{figure}[t]
  \centering
  \begin{subfigure}[t]{0.485\textwidth}
    \centering
    \includegraphics[width=\textwidth]{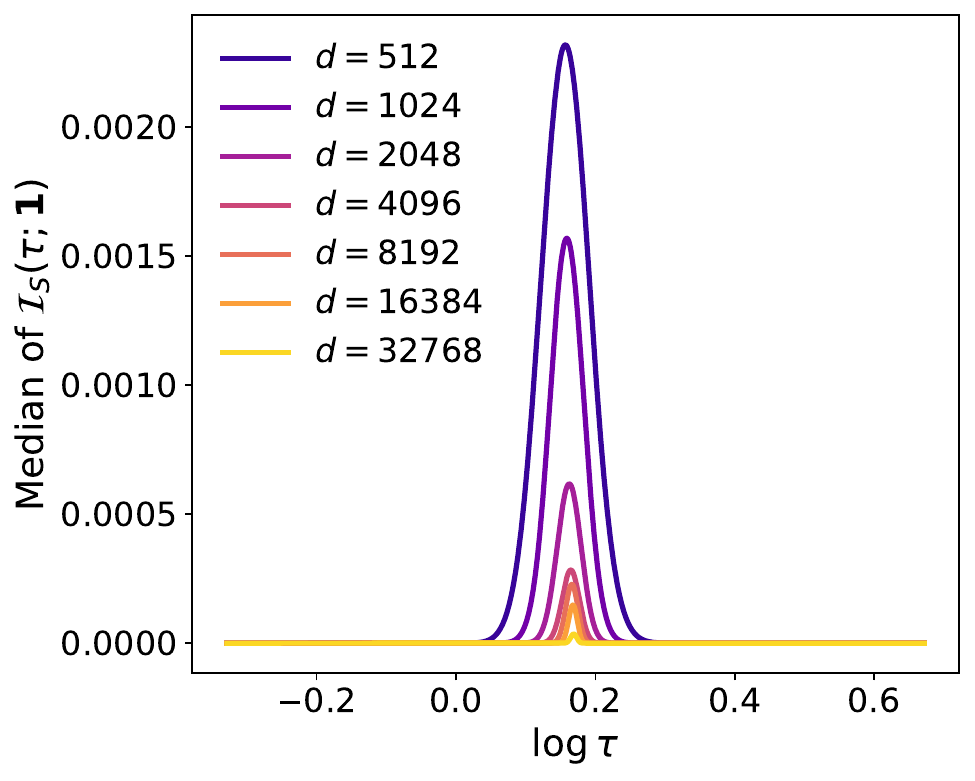}
    \caption{Median of the shell-resolved proxy for
    $\mathcal{I}_S(\tau;\,\mathbf{1})$,
    for $d=512$ to $32768$.}
    \label{fig:synthetic-signal}
  \end{subfigure}
  \hfill
  \begin{subfigure}[t]{0.45\textwidth}
    \centering
    \includegraphics[width=\textwidth]{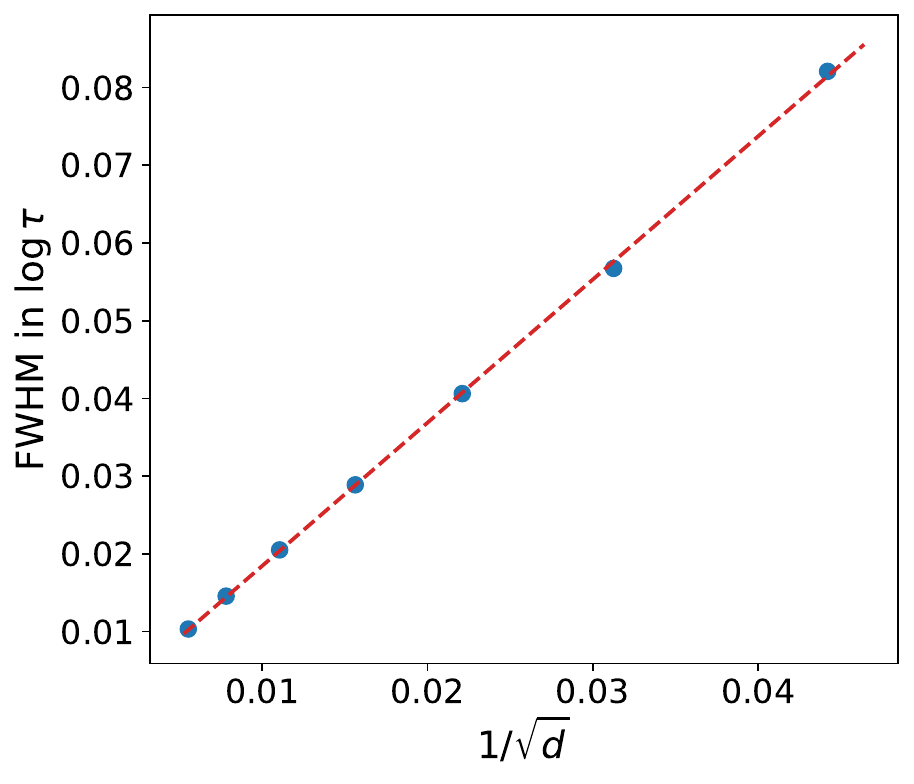}
    \caption{FWHM versus $1/\sqrt{d}$.}
    \label{fig:synthetic-width}
  \end{subfigure}
  \caption{\textbf{Informative window on the hypercube warm-up
  family} ($\kappa=0.2$).
  The proxy concentrates near a single smoothing scale and narrows
  as $d$ grows; the width scales as $1/\sqrt{d}$, consistent with the
  predicted $d^{-1/2}$ rate.
  See Appendix~\ref{app:synthetic} for the precise definition of
  the proxy.}
  \label{fig:main-synthetic}
\end{figure}

\paragraph{Synthetic illustration.}
Figure~\ref{fig:main-synthetic} visualizes the informative-window
phenomenon on the hypercube warm-up family described above.
For a random codebook~$S$ drawn from a random-coding ensemble with
expected size $e^{\kappa d}$ on $\{-1,+1\}^d$ (with $\kappa=0.2$),
consider the local signal at the fixed query
point $x = \mathbf{1}$:
\[
\mathcal{I}_{S}(\tau;\,\mathbf{1})
\;:=\;
\frac{\pi_\tau(\mathbf{1})}{\pi^{\mathrm{null}}_\tau(\mathbf{1})}
\;\cdot\;
\frac{1}{d}\,
\bigl\|s_{\pi,\tau}(\mathbf{1})-s^{\mathrm{null}}_{\tau}(\mathbf{1})
\bigr\|_2^2,
\]
where $\pi_\tau$ is the smoothed density of the mixture defined
by~$S$ and $\pi^{\mathrm{null}}_\tau$ is the smoothed null density
(corresponding to $S = \{-1,+1\}^d$), with $s_{\pi,\tau}$ and
$s^{\mathrm{null}}_\tau$ their respective scores.
This measures the per-coordinate Fisher signal at $x=\mathbf{1}$
between the planted and null distributions at noise level~$\tau$.
When $\mathcal{I}_{S}(\tau;\,\mathbf{1})$ is negligible, a query
at noise level~$\tau$ carries little local planted-versus-null score
signal at this point.

\cref{fig:synthetic-signal} plots a tractable proxy for
$\mathcal{I}_{S}(\tau;\,\mathbf{1})$, obtained by conditioning on
shell occupancies
(see Appendix~\ref{app:synthetic} for details).
The proxy concentrates around a single matching scale and becomes
progressively narrower as $d$ grows.
\cref{fig:synthetic-width} shows that the full width at half
maximum scales as $1/\sqrt{d}$, consistent with the $d^{-1/2}$
prediction from the shell fluctuation argument above.
This narrowing illustrates the mechanism behind the distinct-level
lower bound.

\subsection{Proof Roadmap}
\label{sec:roadmap}

We now state the formal version of \cref{thm:main_simplified} and outline
its proof.

\begin{restatable}{theorem}{MainParameterized}
  \label{thm:main_parameterized}

For every $p>0$, $\rho\in(0,1/4)$, and constant $a\ge1$, there exists
  $c=c(p,\rho,a)>0$ such that the following holds.
  For every $R>0$ and $\gamma\in(0,R/2)$, there exists
  $d_0=d_0(p,\rho,a,R/\gamma)$ such that, for every integer
  $d\ge d_0$ and every $\varepsilon_{\rm err}\in(0,1]$, there exists a distribution
  $\mathcal D$ over pairs
  $(\pi,\{\widehat s_\sigma\}_{\sigma>0})$ on $\mathbb R^d$
  satisfying:
  \begin{enumerate}[label=(\arabic*)]
    \item $\pi=\pi_{\rm pre}*\mathcal N(0,\gamma^2I_d)$ for some
      $\pi_{\rm pre}$ supported on $[-R,R]^d$.

    \item For every $\sigma>0$,
      \[
        \E_{X\sim \pi_\sigma}
        \big[\|\widehat s_\sigma(X)-s_{\pi,\sigma}(X)\|_2^p\big]
        \le \frac{\varepsilon_{\rm err}^p}{\sigma^p}.
      \]

    \item For every $\sigma>0$, both $s_{\pi,\sigma}$ and
      $\widehat s_\sigma$ are globally $L_\sigma$-Lipschitz with
      \[
        L_\sigma\le \frac{1}{\gamma^2+\sigma^2}
        +\frac{5R^2d}{(\gamma^2+\sigma^2)^2}.
      \]

    \item Writing
      \[
        H:=\log(d/\varepsilon_{\rm err})+\log(R/\gamma),
        \qquad
        B_d:=\frac{d\log(R/\gamma)}{\sqrt{dH}+H},
      \]
      every adaptive algorithm that, on every execution path, makes at
      most $d^a$ score queries in total and uses at most $cB_d$ distinct
      smoothing levels (distinct values of $\sigma$) satisfies
      \[
        \Pr_{(\pi,\{\widehat s_\sigma\})\sim\mathcal D}
        \Big[\TV\big(\widehat X,\pi\big)\ge 1-\rho\Big]
        \ge 1-\rho.
      \]
      In particular, the same conclusion holds for every adaptive
      algorithm making at most $cB_d$ score queries in total.
  \end{enumerate}
    \end{restatable}

Appendices~\ref{sec:main-proof-setup}--\ref{sec:completion} prove
\cref{thm:main_parameterized}; Appendix~\ref{sec:nonuniform-weights} extends
it to non-uniform codebook weights of at most $\poly(d)/n$ each.

Given the intuition in Section~\ref{sec:intuition}, formalizing it into a
real query-complexity lower bound still needs to handle several technical
difficulties.  The main issue is that the intuition just shows, for
\emph{each} $x$, that the informative region of $\sigma$ can be small; the
proof needs to construct explicit valid (Lipschitz, accurate) score
functions for each $\sigma$, then show that every fixed $x$ is uninformative
over this family.  Specifically:
\begin{enumerate}[label=(\roman*)]
  \item how to improve the bound by a factor of $\log(R/\gamma)$;
  \item how to formally construct a score oracle that is $L^p$-accurate and
        Lipschitz, yet hides the information;
  \item how to prove that each total noise level $\tau$
        ($\tau^2=\gamma^2+\sigma^2$) gives information about $S$ only for a
        roughly $\sqrt d$-wide range of $\log n$.
\end{enumerate}

In Appendix~\ref{sec:hard-family}, we address (i) by using the product of
$d/2$ blocks as the hard family, where each block contains
$\Theta(R/\gamma)$ well-separated points on a circle of radius $R$, and
take the codebook $S=(Y_1,\dots,Y_n)$ to be $n$ i.i.d.\ uniform points of
this product.  The
admissible range of $\log n$ then has length $\Theta(d\log(R/\gamma))$,
versus $\Theta(d)$ on the hypercube, without losing the Bernstein-type
concentration analysis.

Appendix~\ref{sec:hard-family} also constructs the score estimate of (ii)
by formalizing the intuition of Section~\ref{sec:intuition}: when the true
score is already close in $L^p$ to the null score, just use the null score;
if not, answer with the true score where the codebook mixture's likelihood
ratio against the null is large, and smoothly interpolate to the null score
elsewhere.  The concentration bound below makes this $L^p$-accurate for
every codebook (Lemmas~\ref{lem:mask-properties} and~\ref{lem:oracle-valid}),
and the smoothing, together with the natural smoothness of the true scores,
gives the global Lipschitz bound (Lemma~\ref{lem:oracle-valid}).

Most of the effort goes into proving (iii), which is
Proposition~\ref{prop:window}.  We sketch the proof below.

\paragraph{Dense side.}
Write $f_{S,\tau}$ and $f_{\mathrm{null},\tau}$ for the densities of the
codebook mixture and of the null mixture (uniform over the full product
support) at total noise $\tau$, the analogs of $\pi_\tau$ and its null
counterpart in Section~\ref{sec:intuition}, and define
\[
  \imath_\tau(y,x):=\log\frac{\varphi_{\tau^2}(x-y)}{f_{\mathrm{null},\tau}(x)},
\]
with $\varphi_{\tau^2}$ the density of $\mathcal N(0,\tau^2I_d)$.  Write
$\imath_\tau:=\imath_\tau(Y,X)$ for $Y$ uniform on the full product support
and $X=Y+\tau Z$ with independent $Z\sim\mathcal N(0,I_d)$, so
$I_\tau:=\E[\imath_\tau]$ is the mutual information between $Y$ and $X$.

Consider sampling $X$ from the codebook mixture.  By symmetry, we may assume
that $X$ is generated by the first codeword, $X=Y_1+\tau Z$ with
$Z\sim\mathcal N(0,I_d)$ independent of $S$.  Then
\[
  \frac{f_{S,\tau}(X)}{f_{\mathrm{null},\tau}(X)}
  =\frac1n\sum_{j=1}^n e^{\imath_\tau(Y_j,X)}.
\]
Every term with $j\ne1$ has $\E[e^{\imath_\tau(Y_j,X)}\mid X]=1$, as $Y_j$ is
uniform and independent of $X$, so the only bias away from null is the
correlated term $e^{\imath_\tau(Y_1,X)}/n=e^{\imath_\tau(Y_1,X)-\log n}$,
which is exponentially small once $\log n\gg\imath_\tau(Y_1,X)$.  Therefore, once
$\log n$ is much larger than $I_\tau$, even at the codebook mixture's own
sample points, the mixture has similar density to the null distribution
and the signal cannot be seen.  Heuristically, $ne^{-I_\tau}$ plays the role
of $N_{x,\tau}$ in Section~\ref{sec:intuition}.

To make this quantitative, we need the concentration of $\imath_\tau$.  By
the product structure, $\imath_\tau=\sum_{k=1}^{d/2}\imath_\tau^{(k)}$ with
the block terms i.i.d., and the proof of
Lemma~\ref{lem:information-concentration} shows
\[
  \Pr\bigl[|\imath_\tau^{(k)}-\E\imath_\tau^{(k)}|\ge t\bigr]\le2e^{-t/C}
\]
for an absolute constant $C$, independent of $R/\gamma$ and of $\tau$.
Bernstein's inequality then gives
${\imath_\tau=I_\tau\pm\widetilde O(\sqrt d)}$ with probability
$1-1/\poly(d)$.

To translate the closeness in density at typical points to the $L^p$ score
closeness we want, we use Tweedie's formula and truncate the likelihood
weights: once $\log n$ exceeds $I_\tau+\widetilde O(\sqrt d)$, the true and
null scores are close in $L^2$ (Lemma~\ref{lem:dense-score}), hence in $L^p$
since their difference is bounded, so above this threshold the oracle is
all-null with high probability.

\paragraph{Sparse side.}
For small $n$ the signal is strong but restricted to the codewords'
likelihood neighborhoods, which a blind query misses.  By the same
concentration bound, we can switch on the true score only where
$f_{S,\tau}/f_{\mathrm{null},\tau}\ge e^{I_\tau-\widetilde O(\sqrt d)}/n$
without hurting accuracy, since a sample from the codebook mixture exceeds
this level through its own codeword's term
(Lemma~\ref{lem:mask-properties}); at a blind query point the likelihood
ratio has mean one over the random codebook, so by Markov's inequality the
query crosses the threshold with probability at most
$n e^{-I_\tau+\widetilde O(\sqrt d)}$.  The oracle therefore stays null at
any blind query until $\log n$ reaches $I_\tau-\widetilde O(\sqrt d)$.

By putting the above two together, we can resolve (iii) and show that at
every noise level, every blind query is answered with the null score with
high probability, unless $\log n=I_\tau\pm\widetilde O(\sqrt d)$
(Proposition~\ref{prop:window}).

In Appendix~\ref{sec:completion}, we draw $\log n$ nearly uniformly from a
range of length $\Theta(d\log(R/\gamma))$, so each noise level is
informative with probability only $\widetilde O(1/(\sqrt d\log(R/\gamma)))$.
A standard coupling argument then shows that, unless the algorithm uses
$\widetilde\Omega(\sqrt d\log(R/\gamma))$ distinct levels, it behaves as if
it faced the null oracle (Proposition~\ref{prop:blind-output}), so its
output cannot land on the mass of the random planted target
(Lemma~\ref{lem:tv-exposure}).  This
proves \cref{thm:main_parameterized}; the same witness, rescaled to be
$1$-Lipschitz, gives \cref{thm:wasserstein_simplified}
(Corollary~\ref{cor:wasserstein}).

All of the bounds above are for the $L^p$ model, where the $1/\poly(d)$
oracle accuracy enters the window only through an $O(\log d)$ factor.  For
\cref{thm:psi1_simplified}, the oracle must instead be subexponentially
accurate, and this exponentially small requirement has logarithm
$\Theta(\sqrt d)$; the window widens from $O(\sqrt{d\log d})$ to
$O(\sqrt{d\cdot\sqrt d})$, and the bound drops to $\Omega(d^{1/4})$
(Appendix~\ref{sec:subexponential}).

\section{Related Work}
\label{sec:related}

While a rapidly growing literature studies the iteration complexity of diffusion sampling under various assumptions~\cite{chen2022samplingiseasy, debortoli2022convergence, conforti2023score,lee2022polynomialcomplexity, chen2023probability,lee2023convergence,chen2023improved, chen2023restoration, li2024sharp, li2024accelerating, wu2024stochasticrk, gao2025wasserstein}, lower bounds remain much less developed than convergence analyses.
We discuss two directions that are closest in spirit: 
iteration lower bounds for specific diffusion algorithms and query lower bounds in the classical sampling literature.

\subsection{Iteration Lower Bounds for Diffusion Models}
\label{sec:diffusion_lower}

Existing algorithmic lower bounds typically fix a discretization or a noise schedule.
\cite{jiao2025optimalddpm} show that under a standard DDPM discretization schedule, even with exact scores and a Gaussian target, the KL error of the terminal marginal satisfies $\mathrm{KL} \gtrsim d/T^2$, which forces $T=\Omega(\sqrt d/\varepsilon)$ steps to achieve $\mathrm{KL}\le \varepsilon^2$.
Related Gaussian lower bounds in Wasserstein metrics appear in \cite{gao2025sdewasserstein}.
These results identify barriers for specific solvers, but their analyses are tied to particular discretizations of specific stochastic processes and therefore cannot be generalized to the algorithm-agnostic lower bounds we seek.

It is also worth mentioning a potential route toward dimension-dependent query lower bounds via tensorization.
For a product target $p^{\otimes d}$, KL divergence is additive:
\[
  \KL(p^{\otimes d}\|q^{\otimes d}) = d \cdot \KL(p\|q),
\]
so achieving $\varepsilon$-KL accuracy on $p^{\otimes d}$ requires $(\varepsilon/d)$-KL accuracy on each marginal~\citep{benton2023linearconvergence}.
If one-dimensional score-based sampling required $\mathrm{poly}(1/\varepsilon)$ queries, this could yield dimension-dependent lower bounds even for product targets.
However, recent results~\citep{chen2026highaccuracy} achieve $\mathrm{polylog}(1/\varepsilon)$ query complexity in one dimension under $L^2$-accurate score estimates, ruling out polynomial one-dimensional lower bounds in this oracle model.

\subsection{Query Lower Bounds for Sampling}

A classical oracle model writes the target as $\pi(x)\propto e^{-V(x)}$ and grants pointwise access to $V(x)$ and possibly its derivatives~\citep{chewi2021query1d}.
With first-order access, querying $\nabla V(x)$ is equivalent to querying the exact unsmoothed score
\[
  s_0(x):=\nabla\log\pi(x) = -\nabla V(x).
\]

In the strongly log-concave and smooth setting, \cite{chewi2023querylogconcave} prove dimension-dependent lower bounds that already appear for Gaussian targets.
In particular, for a centered Gaussian with condition number $\kappa$, any algorithm with first-order oracle access needs
\[
  \widetilde{\Omega}\big(\min\{\sqrt{\kappa}\log d,\, d\}\big)
\]
queries to achieve constant accuracy, matching block-Krylov type methods up to logarithmic factors. Sharp lower bounds are also known in one dimension~\citep{chewi2021query1d}.

For general non-log-concavity, it is known that $\exp(\Omega(d))$ queries are needed. 
A typical construction hides probability mass in one of exponentially many well-separated regions of $\R^d$, so that local values and gradients are essentially indistinguishable unless the algorithm queries near the right region~\citep{LRG, he2025nonlogconcave, chewi2022fisherinfo}.

Diffusion-style access changes the picture because the oracle is smoothed.
Querying $s_{\pi,\sigma}=\nabla\log(\pi*\mathcal N(0,\sigma^2I))$ aggregates information through Gaussian smoothing at scale $\sigma$.
Large $\sigma$ can reveal coarse global structure and rule out many candidate regions at once, while small $\sigma$ refines local details.
Therefore, our lower bounds require a different approach, and we show that the bottleneck is identifying an informative smoothing scale $\sigma$.

\section{Applicability to Practical Algorithms}
\label{sec:applicability}

Our oracle model is intended to capture the inference-time information available to a broad class of
practical diffusion and diffusion-adjacent samplers: repeated evaluations of a learned network at
user-chosen states and noise levels, without any additional access to the target distribution $\pi$.

\paragraph{Common network parameterizations are score-equivalent.}
In standard diffusion setups, a network may be parameterized as a denoiser/$x_0$-predictor, a noise predictor, or directly as a score/drift model.
All of these are information-equivalent to a smoothed-score oracle, up to known, noise-dependent affine transformations.
Concretely, under additive Gaussian corruption, the denoiser (posterior mean) $D_\sigma(x)$ and the noise predictor $\varepsilon_\sigma(x)$ satisfy the standard identities
\[
  s_{\pi,\sigma}(x) \;=\; \frac{D_\sigma(x)-x}{\sigma^2},
  \qquad
  \varepsilon_\sigma(x) \;=\; -\sigma\, s_{\pi,\sigma}(x),
\]
and hence any learned approximation of $D_\sigma$ or $\varepsilon_\sigma$ can be converted into a score estimate $\widehat s_\sigma$ by the same formulas.
These conversions are routinely used in diffusion models such as DDPM/DDIM and score-SDE formulations (up to conventional rescalings that depend only on the noise/time parameterization)
\citep{ho2020ddpm,song2021score_sde, karras2022edm, ddim}.
Therefore, up to constant factors, each neural network evaluation can be counted as one oracle query.
(Methods that use higher-order updates or corrector steps may spend multiple network evaluations per
macro-step, but the relevant complexity measure for our results is the \emph{total} number of evaluations.)

\paragraph{SDE/ODE-based diffusion samplers.}
This scope includes both classic reverse-time SDE discretizations and ODE-based solvers.
Concretely, it covers DDPM and its variants, predictor--corrector samplers, the probability flow ODE,
and generic integrators such as DPM-Solver, PNDM, DEIS, UniPC, and EDM-style solvers
\citep{song2021score_sde,lu2022dpmsolver,lu2022dpmsolverpp,liu2022pndm,zhang2022deis,zhao2023unipc,karras2022edm}.
These methods differ in discretization order, step-size selection, and whether the path is formulated
as an SDE or an ODE, but they all obtain information about $\pi$ through the same primitive:
evaluating a learned score (or an equivalent drift derived from it) at adaptively chosen states and noise levels.

\paragraph{Flow matching and rectified flow in the linear Gaussian interpolation setting.}
Our oracle model also captures the standard rectified-flow / flow-matching setup with a Gaussian prior
and linear interpolation. Specifically, let $X_1\sim \pi$ and $X_0\sim\mathcal N(0,I_d)$ be independent and define
\[
  X_t = tX_1 + (1-t)X_0,\qquad t\in(0,1).
\]
Rectified flow defines the velocity field
\[
  v_t(x) := \E[X_1-X_0 \mid X_t=x].
\]
A direct calculation using Tweedie's formula \citep{robbins1956empirical,efron2011tweedie} shows that
$v_t$ is an explicit affine transform of the Gaussian-smoothed score:
\[
  v_t(x)=\frac{x}{t}+\frac{1-t}{t^2}\, s_{\pi,\sigma_t}(x/t),
  \qquad \sigma_t=\frac{1-t}{t}.
\]
Equivalently, for each fixed $t$, querying $v_t(x)$ is invertibly equivalent to querying
$s_{\pi,\sigma_t}$ at the rescaled input $x/t$. Therefore, any sampler whose inference-time access to $\pi$
is only through evaluations of learned approximations of such velocity fields inherits the same query lower bounds
\citep{lipman2022flowmatching,liu2022rectified, benton2023error, fukumizu2025flow}.

\paragraph{What is not covered.}
Approaches such as progressive distillation
\citep{salimans2022progressive} or consistency
models \citep{song2023consistency} may fall outside this model:
their training procedures can encode additional global information
about $\pi$ into a new mapping, so that each inference-time
evaluation is no longer a score-equivalent query.
For hybrid systems that combine a partially distilled backbone
with residual score-correction steps, applicability depends on
whether the correction steps are the sole source of information
about $\pi$ at inference time; if so, the lower bound applies to
the number of such correction queries.

\section{Conclusion}
In this work, we establish the first information-theoretic lower bounds for diffusion sampling. We prove that for $d$-dimensional distributions, acquiring a nontrivial sample using $L^p$-accurate score estimates requires $\widetilde{\Omega}(\sqrt{d})$ adaptive score queries, even under favorable assumptions of smoothness and polynomial estimation accuracy.
These results show that the dependence of the computational cost on dimension is a fundamental barrier intrinsic to the score-based sampling paradigm, rather than a limitation of current solvers.

At a technical level, our proof shows that, within any
polynomial total-query budget, successful sampling requires querying
$\widetilde{\Omega}(\sqrt{d})$ distinct noise levels, providing
an information-theoretic rationale for the multiscale noise
schedules used in practice.
The key mechanism is that the score signal is localized within a
narrow window of noise levels; outside this window, oracle
responses carry essentially no information about the target.
This means that the bottleneck for few-step sampling is not
discretization error but the number of noise levels queried.

Our work raises interesting open directions for future research:
\begin{enumerate}
\item \textbf{Tightening the Bound.}
There remains a gap between our $\widetilde\Omega(\sqrt{d})$ lower
bound and the best known $\widetilde O(d)$ upper bounds.
Determining the precise query complexity is a major open question.
On the lower-bound side, our reduction to hypothesis testing
suggests a natural avenue: directly lower-bounding the difficulty
of producing a sample, rather than extracting a single bit.
On the upper-bound side, the noise-level scanning bottleneck
identified here may inform the design of faster algorithms.
    \item \textbf{Beyond the Smoothed-Score Oracle.}
Our lower bounds apply when inference-time access to the target is only through repeated evaluations of a Gaussian-smoothed score, or an explicitly equivalent field as in rectified flow.
A natural next step is to extend lower bounds to paradigms that change this primitive, including flow matching with non-Gaussian priors or alternative couplings, and distilled generators that produce samples in very few steps.

     \item \textbf{Stronger Oracle and Instance Constraints.}
Our construction is adversarial within Assumptions~\ref{ass:bounded} and~\ref{ass:Lp}, with no restriction on how the score is realized. A natural refinement is to require $\widehat s_\sigma$, or even $s_{\pi,\sigma}$ itself, to lie in the function class of polynomial-sized neural networks, or to grant access to the \emph{exact} smoothed score in place of an $L^p$-accurate estimate.
Both restrictions remove substantial degrees of freedom in our construction
and whether the $\widetilde\Omega(\sqrt d)$ barrier persists under such constraints is open.

    \item \textbf{Real-World Structure.} The empirical success of few-step sampling suggests that real-world data possess structures stronger than generic smoothness. Identifying specific geometric properties, such as low intrinsic dimension, that enable practical algorithms to break the worst-case $\sqrt{d}$ barrier is a promising direction for bridging theory and practice.
\end{enumerate}

\section*{Acknowledgments}

ZX thanks Chihao Zhang for a helpful discussion on query lower bounds for log-concave sampling.
This work is supported by the NSF AI Institute for Foundations of Machine Learning (IFML).

\bibliographystyle{alpha}
\bibliography{ref_arxiv_2026}

@article{efron2011tweedie,
  title        = {Tweedie's Formula and Selection Bias},
  author       = {Efron, Bradley},
  journal      = {Journal of the American Statistical Association},
  volume       = {106},
  number       = {496},
  pages        = {1602--1614},
  year         = {2011},
  doi          = {10.1198/jasa.2011.tm11181}
}

@inproceedings{robbins1956empirical,
  title        = {An Empirical {B}ayes Approach to Statistics},
  author       = {Robbins, Herbert},
  booktitle    = {Proceedings of the Third Berkeley Symposium on Mathematical Statistics and Probability},
  volume       = {1},
  pages        = {157--163},
  publisher    = {University of California Press},
  address      = {Berkeley, CA},
  year         = {1956},
  editor       = {Neyman, Jerzy}
}

@book{coverthomas2006,
  title     = {Elements of Information Theory},
  author    = {Cover, Thomas M. and Thomas, Joy A.},
  edition   = {Second},
  year      = {2006},
  publisher = {Wiley-Interscience}
}

@book{vershynin2018,
  title     = {High-Dimensional Probability: An Introduction with Applications in Data Science},
  author    = {Vershynin, Roman},
  year      = {2018},
  publisher = {Cambridge University Press}
}

@article{hyvarinen2005score,
  title   = {Estimation of Non-Normalized Statistical Models by Score Matching},
  author  = {Hyv{\"a}rinen, Aapo},
  journal = {Journal of Machine Learning Research},
  volume  = {6},
  number  = {24},
  pages   = {695--709},
  year    = {2005}
}

@article{vincent2011connection,
  title={A Connection Between Score Matching and Denoising Autoencoders},
  author={Vincent, Pascal},
  journal={Neural Computation},
  volume={23},
  number={7},
  pages={1661--1674},
  year={2011}
}

@inproceedings{song2019sgm,
  title     = {Generative Modeling by Estimating Gradients of the Data Distribution},
  author    = {Song, Yang and Ermon, Stefano},
  booktitle = {Advances in Neural Information Processing Systems},
  year      = {2019}
}

@inproceedings{ho2020ddpm,
  title     = {Denoising Diffusion Probabilistic Models},
  author    = {Ho, Jonathan and Jain, Ajay and Abbeel, Pieter},
  booktitle = {Advances in Neural Information Processing Systems},
  year      = {2020}
}

@inproceedings{nichol2021improved,
  title     = {Improved Denoising Diffusion Probabilistic Models},
  author    = {Nichol, Alex and Dhariwal, Prafulla},
  booktitle = {Proceedings of the 38th International Conference on Machine Learning},
  year      = {2021}
}

@inproceedings{song2021score_sde,
  title     = {Score-Based Generative Modeling through Stochastic Differential Equations},
  author    = {Song, Yang and Sohl-Dickstein, Jascha and Kingma, Diederik P. and Kumar, Abhishek and Ermon, Stefano and Poole, Ben},
  booktitle = {International Conference on Learning Representations},
  year      = {2021},
  note      = {arXiv:2011.13456}
}

@inproceedings{karras2022edm,
  title     = {Elucidating the Design Space of Diffusion-Based Generative Models},
  author    = {Karras, Tero and Aittala, Miika and Aila, Timo and Laine, Samuli},
  booktitle = {Advances in Neural Information Processing Systems},
  year      = {2022}
}

@inproceedings{lu2022dpmsolver,
  author = {Lu, Cheng and Zhou, Yuhao and Bao, Fan and Chen, Jianfei and Li, Chongxuan and Zhu, Jun},
  title = {{DPM-Solver}: A Fast {ODE} Solver for Diffusion Probabilistic Model Sampling in Around 10 Steps},
  booktitle = {Advances in Neural Information Processing Systems},
  year = {2022}
}

@article{lu2022dpmsolverpp,
  author = {Lu, Cheng and Zhou, Yuhao and Bao, Fan and Chen, Jianfei and Li, Chongxuan and Zhu, Jun},
  title = {{DPM-Solver++}: Fast Solver for Guided Sampling of Diffusion Probabilistic Models},
  journal = {Machine Intelligence Research},
  year = {2025},
  volume = {22},
  number = {4},
  pages = {730--751},
  doi = {10.1007/s11633-025-1562-4},
  url = {https://doi.org/10.1007/s11633-025-1562-4},
  note = {arXiv:2211.01095},
}

@inproceedings{liu2022pndm,
  author = {Liu, Luping and Ren, Yi and Lin, Zhijie and Zhao, Zhou},
  title = {Pseudo Numerical Methods for Diffusion Models on Manifolds},
  booktitle = {International Conference on Learning Representations},
  year = {2022},
  url = {https://openreview.net/forum?id=PlKWVd2yBkY},
  note = {arXiv:2202.09778},
}

@inproceedings{zhang2022deis,
  author = {Zhang, Qinsheng and Chen, Yongxin},
  title = {Fast Sampling of Diffusion Models with Exponential Integrator},
  booktitle = {International Conference on Learning Representations},
  year = {2023},
  url = {https://openreview.net/forum?id=Loek7hfb46P},
  note = {arXiv:2204.13902},
}

@inproceedings{zhao2023unipc,
  author = {Zhao, Wenliang and Bai, Lujia and Rao, Yongming and Zhou, Jie and Lu, Jiwen},
  title = {{UniPC}: A Unified Predictor-Corrector Framework for Fast Sampling of Diffusion Models},
  booktitle = {Advances in Neural Information Processing Systems},
  year = {2023},
  note = {arXiv:2302.04867},
}

@inproceedings{lipman2022flowmatching,
  author = {Lipman, Yaron and Chen, Ricky T. Q. and Ben-Hamu, Heli and Nickel, Maximilian and Le, Matt},
  title = {Flow Matching for Generative Modeling},
  booktitle = {International Conference on Learning Representations},
  year = {2023},
  url = {https://openreview.net/forum?id=PqvMRDCJT9t},
  note = {arXiv:2210.02747},
}

@inproceedings{liu2022rectified,
  title={Flow Straight and Fast: Learning to Generate and Transfer Data with Rectified Flow},
  author={Liu, Xingchao and Gong, Chengyue and Liu, Qiang},
  booktitle={International Conference on Learning Representations},
  year={2023},
  note={arXiv:2209.03003}
}

@inproceedings{salimans2022progressive,
  title     = {Progressive Distillation for Fast Sampling of Diffusion Models},
  author    = {Salimans, Tim and Ho, Jonathan},
  booktitle = {International Conference on Learning Representations},
  year      = {2022},
  note      = {arXiv:2202.00512}
}

@inproceedings{song2023consistency,
  author = {Song, Yang and Dhariwal, Prafulla and Chen, Mark and Sutskever, Ilya},
  title = {Consistency Models},
  booktitle = {Proceedings of the 40th International Conference on Machine Learning},
  year = {2023},
  volume = {202},
  pages = {32211--32252},
  series = {Proceedings of Machine Learning Research},
  url = {https://proceedings.mlr.press/v202/song23a.html},
  note = {arXiv:2303.01469},
}

@inproceedings{chen2022samplingiseasy,
  title     = {Sampling is as Easy as Learning the Score: Theory for Diffusion Models with Minimal Data Assumptions},
  author    = {Chen, Sitan and Chewi, Sinho and Li, Jerry and Li, Yuanzhi and Salim, Adil and Zhang, Anru R.},
  booktitle = {International Conference on Learning Representations},
  year      = {2023},
  note      = {arXiv:2209.11215}
}

@inproceedings{benton2023linearconvergence,
  title     = {Nearly $d$-Linear Convergence Bounds for Diffusion Models via Stochastic Localization},
  author    = {Benton, Joe and De Bortoli, Valentin and Doucet, Arnaud and Deligiannidis, George},
  booktitle = {International Conference on Learning Representations},
  year      = {2024},
  note      = {arXiv:2308.03686}
}

@article{li2024odt,
  title={{O(d/T)} Convergence Theory for Diffusion Probabilistic Models under Minimal Assumptions},
  author={Li, Gen and Yan, Yuling},
  journal={Journal of Machine Learning Research},
  volume={26},
  number={292},
  pages={1--55},
  year={2025},
  note={arXiv:2409.18959}
}

@article{conforti2023score,
  title={{KL} Convergence Guarantees for Score Diffusion Models under Minimal Data Assumptions},
  author={Conforti, Giovanni and Durmus, Alain and Gentiloni Silveri, Marta},
  journal={SIAM Journal on Mathematics of Data Science},
  volume={7},
  number={1},
  pages={86--109},
  year={2025},
  doi={10.1137/23M1613670},
  note={arXiv:2308.12240}
}

@inproceedings{gupta2024randomized,
  author = {Gupta, Shivam and Cai, Linda and Chen, Sitan},
  title = {Faster Diffusion Sampling with Randomized Midpoints: Sequential and Parallel},
  booktitle = {International Conference on Learning Representations},
  year = {2025},
  url = {https://proceedings.iclr.cc/paper_files/paper/2025/hash/f30307ac840b88f86f4ab5761b2d6595-Abstract-Conference.html},
  note = {arXiv:2406.00924},
}

@inproceedings{li2024accelerating,
  title     = {Accelerating Convergence of Score-Based Diffusion Models, Provably},
  author    = {Li, Gen and Huang, Yu and Efimov, Timofey and Wei, Yuting and Chi, Yuejie and Chen, Yuxin},
  booktitle = {Proceedings of the 41st International Conference on Machine Learning},
  year      = {2024},
  note      = {arXiv:2403.03852}
}

@article{wu2024stochasticrk,
  title={Stochastic {R}unge--{K}utta Methods: Provable Acceleration of Diffusion Models},
  author={Wu, Yuchen and Chen, Yuxin and Wei, Yuting},
  journal={arXiv preprint arXiv:2410.04760},
  year={2024}
}

@article{jiaoli2024instancedependent,
  title={Instance-dependent Convergence Theory for Diffusion Models},
  author={Jiao, Yuchen and Li, Gen},
  journal={arXiv preprint arXiv:2410.13738v2},
  year={2025}
}

@article{zhang2025sublinear,
  title={Sublinear iterations can suffice even for {DDPMs}},
  author={Zhang, Matthew S. and Huan, Stephen and Huang, Jerry and Boffi, Nicholas M. and Chen, Sitan and Chewi, Sinho},
  journal={arXiv preprint arXiv:2511.04844},
  year={2025}
}

@article{jiao2025optimalddpm,
  title   = {Optimal Convergence Analysis of {DDPM} for General Distributions},
  author  = {Jiao, Yuchen and Zhou, Yuchen and Li, Gen},
  journal = {arXiv preprint arXiv:2510.27562},
  year    = {2025}
}

@inproceedings{gao2025wasserstein,
  author = {Gao, Xuefeng and Zhu, Lingjiong},
  title = {Convergence Analysis for General Probability Flow {ODE}s of Diffusion Models in {W}asserstein Distances},
  booktitle = {International Conference on Artificial Intelligence and Statistics (AISTATS)},
  year = {2025},
  note = {arXiv:2401.17958},
  pages = {1009--1017},
  series = {Proceedings of Machine Learning Research},
  volume = {258},
}

@inproceedings{gatmiry2026highaccuracy,
  title={High-accuracy and Dimension-free Sampling with Diffusions},
  author={Gatmiry, Khashayar and Chen, Sitan and Salim, Adil},
  booktitle={International Conference on Machine Learning},
  year={2026},
  note={arXiv:2601.10708}
}

@inproceedings{lee2022polynomialcomplexity,
  title     = {Convergence for score-based generative modeling with polynomial complexity},
  author    = {Lee, Holden and Lu, Jianfeng and Tan, Yixin},
  booktitle = {Advances in Neural Information Processing Systems},
  year      = {2022},
  note      = {arXiv:2206.06227}
}

@inproceedings{chewi2021query1d,
  title     = {The query complexity of sampling from strongly log-concave distributions in one dimension},
  author    = {Chewi, Sinho and Gerber, Patrik and Lu, Chen and Le Gouic, Thibaut and Rigollet, Philippe},
  booktitle = {Conference on Learning Theory},
  year      = {2022},
  note      = {arXiv:2105.14163}
}

@article{chewi2023querylogconcave,
  author = {Chewi, Sinho and de Dios Pont, Jaume and Li, Jerry and Lu, Chen and Narayanan, Shyam},
  title = {Query Lower Bounds for Log-concave Sampling},
  journal = {Journal of the ACM},
  year = {2024},
  volume = {71},
  number = {4},
  pages = {29:1--29:42},
  doi = {10.1145/3673651},
  url = {https://doi.org/10.1145/3673651},
  note = {Preliminary version in FOCS 2023; arXiv:2304.02599},
}

@inproceedings{chewi2022fisherinfo,
  title     = {Fisher information lower bounds for sampling},
  author    = {Chewi, Sinho and Gerber, Patrik and Lee, Holden and Lu, Chen},
  booktitle = {Algorithmic Learning Theory},
  year      = {2023},
  note      = {arXiv:2210.02482}
}

@inproceedings{he2025nonlogconcave,
  author = {He, Yuchen and Zhang, Chihao},
  title = {On the query complexity of sampling from non-log-concave distributions (extended abstract)},
  booktitle = {Proceedings of Thirty Eighth Conference on Learning Theory},
  year = {2025},
  volume = {291},
  pages = {2786--2787},
  series = {Proceedings of Machine Learning Research},
  publisher = {PMLR},
  editor = {Haghtalab, Nika and Moitra, Ankur},
  url = {https://proceedings.mlr.press/v291/he25a.html},
  note = {Full version: arXiv:2502.06200},
}

@inproceedings{sohl2015diffusion,
  title     = {Deep Unsupervised Learning using Nonequilibrium Thermodynamics},
  author    = {Sohl-Dickstein, Jascha and Weiss, Eric and Maheswaranathan, Niru and Ganguli, Surya},
  booktitle = {Proceedings of the 32nd International Conference on Machine Learning},
  series    = {Proceedings of Machine Learning Research},
  volume    = {37},
  pages     = {2256--2265},
  year      = {2015}
}

@inproceedings{
li2025improved,
title={Improved Convergence Rate for Diffusion Probabilistic Models},
author={Gen Li and Yuchen Jiao},
booktitle={The Thirteenth International Conference on Learning Representations},
year={2025},
url={https://openreview.net/forum?id=SOd07Qxkw4}
}

@inproceedings{LRG,
 author = {Ge, Rong and Lee, Holden and Risteski, Andrej},
 booktitle = {Advances in Neural Information Processing Systems},
 editor = {S. Bengio and H. Wallach and H. Larochelle and K. Grauman and N. Cesa-Bianchi and R. Garnett},
 pages = {},
 publisher = {Curran Associates, Inc.},
 title = {Beyond Log-concavity: Provable Guarantees for Sampling Multi-modal Distributions using Simulated Tempering {L}angevin {M}onte {C}arlo},
 url = {https://proceedings.neurips.cc/paper_files/paper/2018/file/c6ede20e6f597abf4b3f6bb30cee16c7-Paper.pdf},
 volume = {31},
 year = {2018}
}

@inproceedings{
ddim,
title={Denoising Diffusion Implicit Models},
author={Jiaming Song and Chenlin Meng and Stefano Ermon},
booktitle={International Conference on Learning Representations},
year={2021},
url={https://openreview.net/forum?id=St1giarCHLP}
}

@inproceedings{
fukumizu2025flow,
title={Flow matching achieves almost minimax optimal convergence},
author={Kenji Fukumizu and Taiji Suzuki and Noboru Isobe and Kazusato Oko and Masanori Koyama},
booktitle={The Thirteenth International Conference on Learning Representations},
year={2025},
url={https://openreview.net/forum?id=2OMyAFjiJJ}
}

@inproceedings{improvedSample,
 author = {Gupta, Shivam and Parulekar, Aditya and Price, Eric and Xun, Zhiyang},
 booktitle = {Advances in Neural Information Processing Systems},
 doi = {10.52202/079017-1296},
 editor = {A. Globerson and L. Mackey and D. Belgrave and A. Fan and U. Paquet and J. Tomczak and C. Zhang},
 pages = {40976--41012},
 publisher = {Curran Associates, Inc.},
 title = {Improved Sample Complexity Bounds for Diffusion Model Training},
 url = {https://proceedings.neurips.cc/paper_files/paper/2024/file/480e563034dbb6d1dd622d8eab7d129b-Paper-Conference.pdf},
 volume = {37},
 year = {2024}
}

@InProceedings{L41,
  title = 	 {Adaptivity of Diffusion Models to Manifold Structures},
  author =       {Tang, Rong and Yang, Yun},
  booktitle = 	 {Proceedings of The 27th International Conference on Artificial Intelligence and Statistics},
  pages = 	 {1648--1656},
  year = 	 {2024},
  editor = 	 {Dasgupta, Sanjoy and Mandt, Stephan and Li, Yingzhen},
  volume = 	 {238},
  series = 	 {Proceedings of Machine Learning Research},
  month = 	 {02--04 May},
  publisher =    {PMLR},
  url = 	 {https://proceedings.mlr.press/v238/tang24a.html}
}

@inproceedings{
xun2025posterior,
title={Posterior Sampling by Combining Diffusion Models with Annealed {L}angevin Dynamics},
author={Zhiyang Xun and Shivam Gupta and Eric Price},
booktitle={The Thirty-ninth Annual Conference on Neural Information Processing Systems},
year={2025},
url={https://openreview.net/forum?id=ARZiMmb619}
}

@article{li2024sharp,
  title={A sharp convergence theory for the probability flow {ODE}s of diffusion models},
  author={Li, Gen and Wei, Yuting and Chi, Yuejie and Chen, Yuxin},
  journal={arXiv preprint arXiv:2408.02320},
  year={2024}
}

@article{debortoli2022convergence,
title={Convergence of denoising diffusion models under the manifold hypothesis},
author={De Bortoli, Valentin},
journal={Transactions on Machine Learning Research},
issn={2835-8856},
year={2022},
url={https://openreview.net/forum?id=MhK5aXo3gB},
note={Expert Certification}
}

@inproceedings{chen2023probability,
  title={The probability flow {ODE} is provably fast},
  author={Chen, Sitan and Chewi, Sinho and Lee, Holden and Li, Yuanzhi and Lu, Jianfeng and Salim, Adil},
  booktitle={Advances in Neural Information Processing Systems},
  volume={36},
  pages={68552--68575},
  year={2023}
}

@inproceedings{lee2023convergence,
  title={Convergence of score-based generative modeling for general data distributions},
  author={Lee, Holden and Lu, Jianfeng and Tan, Yixin},
  booktitle={International Conference on Algorithmic Learning Theory},
  pages={946--985},
  year={2023},
  organization={PMLR}
}

@inproceedings{chen2023improved,
  title={Improved analysis of score-based generative modeling: User-friendly bounds under minimal smoothness assumptions},
  author={Chen, Hongrui and Lee, Holden and Lu, Jianfeng},
  booktitle={International Conference on Machine Learning},
  pages={4735--4763},
  year={2023},
  organization={PMLR}
}

@inproceedings{chen2023restoration,
  title={Restoration-degradation beyond linear diffusions: A non-asymptotic analysis for {DDIM}-type samplers},
  author={Chen, Sitan and Daras, Giannis and Dimakis, Alex},
  booktitle={International Conference on Machine Learning},
  pages={4462--4484},
  year={2023},
  organization={PMLR}
}

@article{benton2023error,
  title={Error Bounds for Flow Matching Methods},
  author={Benton, Joe and Deligiannidis, George and Doucet, Arnaud},
  journal={Transactions on Machine Learning Research},
  year={2024},
  note={arXiv:2305.16860}
}

@inproceedings{chen2026highaccuracy,
  title={High-accuracy sampling for diffusion models and log-concave distributions},
  author={Chen, Fan and Chewi, Sinho and Daskalakis, Constantinos and Rakhlin, Alexander},
  booktitle={International Conference on Machine Learning},
  year={2026},
  note={arXiv:2602.01338}
}

@article{dytso-poor-shamai,
  author  = {Dytso, Alex and Poor, H. Vincent and {Shamai (Shitz)}, Shlomo},
  title   = {Conditional Mean Estimation in {G}aussian Noise: A Meta Derivative Identity with Applications},
  journal = {IEEE Transactions on Information Theory},
  volume  = {69},
  number  = {3},
  pages   = {1883--1898},
  year    = {2023},
  doi     = {10.1109/TIT.2022.3216012}
}

@inproceedings{arsenyan2025wasserstein,
  title     = {Assessing the Quality of Denoising Diffusion Models in {W}asserstein Distance: Noisy Score and Optimal Bounds},
  author    = {Arsenyan, Vahan and Vardanyan, Elen and Dalalyan, Arnak},
  booktitle = {Advances in Neural Information Processing Systems},
  volume    = {38},
  year      = {2025},
  note      = {arXiv:2506.09681}
}

@inproceedings{gentilonisilveri2025beyond,
  title     = {Beyond Log-Concavity and Score Regularity: Improved Convergence Bounds for Score-Based Generative Models in {$W_2$}-distance},
  author    = {Gentiloni Silveri, Marta and Ocello, Antonio},
  booktitle = {International Conference on Machine Learning},
  series    = {Proceedings of Machine Learning Research},
  volume    = {267},
  pages     = {55631--55656},
  year      = {2025},
  note      = {arXiv:2501.02298}
}

@article{gao2025sdewasserstein,
  title   = {{W}asserstein Convergence Guarantees for a General Class of Score-Based Generative Models},
  author  = {Gao, Xuefeng and Nguyen, Hoang M. and Zhu, Lingjiong},
  journal = {Journal of Machine Learning Research},
  volume  = {26},
  number  = {43},
  pages   = {1--54},
  year    = {2025},
  note    = {arXiv:2311.11003}
}
\onecolumn
\appendix

\renewcommand{\eps}{\varepsilon_{\rm err}}

\section{Setup}
\label{sec:main-proof-setup}

Sections~\ref{sec:main-proof-setup}--\ref{sec:completion} prove
Theorem~\ref{thm:main_parameterized}, following the roadmap in
Section~\ref{sec:roadmap}.
Extensions to non-uniform weights, subexponential errors, and
Wasserstein distance follow in
Sections~\ref{sec:nonuniform-weights}--\ref{sec:wasserstein}.
Appendix~\ref{app:synthetic} describes the synthetic illustration in
Figure~\ref{fig:main-synthetic}.

As in the main text, for a target law \(\pi\) and \(\sigma>0\), write
\[
   \pi_\sigma:=\pi*\cN(0,\sigma^2I_d),
   \qquad
   s_{\pi,\sigma}:=\nabla\log\pi_\sigma .
\]
For an algorithm \(\mathcal A\) interacting with \((\pi,\widehat s)\),
write \(\widehat X_{\pi,\widehat s}^{\mathcal A}\) for its output.

It suffices to prove Theorem~\ref{thm:main_parameterized} for \(p\ge2\):
for \(0<p<2\), the \(p=2\) construction satisfies the required moment
bound by Jensen's inequality.
Fix the parameters \(p,\rho,a,R,\gamma,d,\eps\) from that theorem,
with \(p\ge2\), and write
\[
   H:=\log(d/\eps)+\log(R/\gamma),
   \qquad
   B_d:=\frac{d\log(R/\gamma)}{\sqrt{dH}+H}.
\]

Throughout the analytic estimates, \(c,C>0\) denote universal
constants that may change from line to line; subscripts indicate
the permitted parameter dependence.
Fix throughout
\[
   \delta_d:=\frac{\rho^2}{8d^a},
   \qquad
   w:=A_{p,\rho,a}\bigl(\sqrt{dH}+H\bigr),
\]
where \(A_{p,\rho,a}>0\) will be chosen sufficiently large.

\section{The target--oracle family}
\label{sec:hard-family}

\subsection{Product-circle mixtures and the common null reference}

Let
\[
   b:=\lfloor d/2\rfloor,
   \qquad
   M:=\left\lceil\frac{\pi R}{\gamma}\right\rceil.
\]
Place \(M\) equally spaced points on the circle of radius \(R\):
\[
   v_k:=R\bigl(\cos(2\pi k/M),\sin(2\pi k/M)\bigr),
   \qquad 0\le k<M.
\]
Let
\[
   V:=\{v_0,\ldots,v_{M-1}\}^{b}
      \times\{0\}^{d-2b}
      \subset[-R,R]^d,
   \qquad
   \nu_{\mathrm{null}}:=\Unif(V).
\]
There are \(b\) planar blocks and, for odd \(d\), one fixed coordinate.
For an integer \(n\ge1\), draw an i.i.d. codebook
\[
   S=(Y_1,\ldots,Y_n),
   \qquad Y_i\stackrel{\mathrm{iid}}\sim\nu_{\mathrm{null}},
\]
and let
\[
   \nu_S:=\frac1n\sum_{i=1}^n\delta_{Y_i}.
\]
At total noise level \(\tau>0\), define the planted and null
Gaussian mixtures
\[
   \pi_{S,\tau}:=\nu_S*\cN(0,\tau^2I_d),
   \qquad
   \pi_{\mathrm{null},\tau}
   :=\nu_{\mathrm{null}}*\cN(0,\tau^2I_d).
\]
Their densities are denoted by \(f_{S,\tau}\) and
\(f_{\mathrm{null},\tau}\), and their scores by \(s_{S,\tau}\) and
\(s_{\mathrm{null},\tau}\).  Here \(\tau\) denotes total noise;
Section~\ref{sec:noise-translation} relates it to the added smoothing
\(\sigma\).

The full product support has
\[
   \log|V|=b\log M
          \asymp d\log(R/\gamma),
\]
which determines the range of log-codebook sizes used below.

The null mixture is common to every codebook.  For
\(y\in V\) and \(x\in\R^d\), define the log-likelihood ratio of one
codeword against the null mixture by
\[
   \imath_\tau(y,x)
   :=\log\frac{\varphi_{\tau^2}(x-y)}
                   {f_{\mathrm{null},\tau}(x)},
\]
where \(\varphi_{\tau^2}\) is the density of
\(\cN(0,\tau^2I_d)\).  If \(Y\sim\nu_{\mathrm{null}}\) and
\(Z\sim\cN(0,I_d)\) are independent, write
\[
   \imath_\tau:=\imath_\tau(Y,Y+\tau Z),
   \qquad
   I_\tau:=\E\imath_\tau=I(Y;Y+\tau Z).
\]
Define the empirical log-likelihood ratio
\[
   r_{S,\tau}(x)
   :=\log\frac{f_{S,\tau}(x)}{f_{\mathrm{null},\tau}(x)}.
\]
Then
\begin{equation}\label{eq:empirical-likelihood}
   e^{r_{S,\tau}(x)}
   =\frac1n\sum_{i=1}^n e^{\imath_\tau(Y_i,x)},
   \qquad
   \nabla r_{S,\tau}=s_{S,\tau}-s_{\mathrm{null},\tau}.
\end{equation}
In particular, for every fixed \(x\),
\begin{equation}\label{eq:mean-empirical-likelihood}
   \E_S e^{r_{S,\tau}(x)}=1.
\end{equation}

\subsection{Score bounds and information concentration}

\begin{lemma}[Gaussian smoothing and score bounds]
\label{lem:smoothing}
Let \(Y\sim\nu\) be supported on \([-R,R]^d\), let
\(Z\sim\cN(0,I_d)\) be independent of \(Y\), and set
\(X:=Y+\tau Z\).  Write \(f_{\nu,\tau}\) for the density of \(X\)
and \(s_{\nu,\tau}:=\nabla\log f_{\nu,\tau}\) for its score.  Then
\[
   s_{\nu,\tau}(x)
   =\frac{\E[Y\mid X=x]-x}{\tau^2}
\]
and
\[
   \nabla s_{\nu,\tau}(x)
   =-\frac1{\tau^2}I_d
    +\frac1{\tau^4}\operatorname{Cov}(Y\mid X=x).
\]
Consequently, for any two such input laws \(\nu_1,\nu_2\),
\(x\in\R^d\),
\[
   \|s_{\nu_1,\tau}(x)-s_{\nu_2,\tau}(x)\|_2
   \le\frac{2R\sqrt d}{\tau^2},
\]
and
\[
   \Lip(s_{\nu,\tau})
   \le\frac1{\tau^2}+\frac{R^2d}{\tau^4}.
\]
\end{lemma}

\begin{proof}
The two displayed identities are the vector Tweedie and
Hatsell--Nolte identities \cite{dytso-poor-shamai}.  Since every
posterior mean has norm at most \(R\sqrt d\), two posterior means
differ by at most \(2R\sqrt d\).  Moreover,
\[
   \|\operatorname{Cov}(Y\mid X=x)\|_{\mathrm{op}}\le R^2d,
\]
which gives the score Lipschitz bound.
\end{proof}

\begin{lemma}[Noise-uniform concentration of information density]
\label{lem:information-concentration}
There is a universal constant \(c>0\) such that, for every
\(\tau>0\) and every \(u\ge0\),
\[
   \Prob\left[
      |\imath_\tau-I_\tau|\ge u
   \right]
   \le
   2\exp\left[
      -c\min\left\{\frac{u^2}{d},u\right\}
   \right].
\]
\end{lemma}

\begin{proof}
The channel factors across the \(b\) planar blocks, and the possible
last coordinate cancels from the likelihood ratio.  Thus
\(\imath_\tau\) is a sum of independent block contributions.
By rotational symmetry, a block contribution has the distribution
\[
   \log M-\log\mathcal Z_\tau(Z),\qquad Z\sim\cN(0,I_2),
\]
where
\[
   \mathcal Z_\tau(z):=\sum_{k=0}^{M-1}
   \exp\left(
      -\frac{\|v_0-v_k\|_2^2}{2\tau^2}
      -\frac{\langle z,v_0-v_k\rangle}{\tau}
   \right).
\]
The circle chord bounds give
\(\|v_k-v_0\|_2\asymp\gamma\min\{k,M-k\}\).  Hence the number
of points within distance \(\tau\) of \(v_0\) satisfies
\[
   K_\tau:=\#\{k:\|v_k-v_0\|_2\le\tau\}
   \asymp\min\{M,1+\tau/\gamma\},
\]
and comparison with a Gaussian integral gives
\[
   \sum_k e^{-\|v_k-v_0\|_2^2/(4\tau^2)}\le C K_\tau.
\]
Keeping the \(K_\tau\) nearby terms for the lower bound and completing
the square for the upper bound yields
\[
   K_\tau e^{-1/2-\|z\|_2}
   \le\mathcal Z_\tau(z)
   \le e^{\|z\|_2^2}
      \sum_k e^{-\|v_k-v_0\|_2^2/(4\tau^2)}
   \le C K_\tau e^{\|z\|_2^2}.
\]
Thus \(|\log\mathcal Z_\tau(Z)-\log K_\tau|
\le C(1+\|Z\|_2^2)\), so each centered block contribution has
a universal subexponential norm.  Bernstein's inequality
\cite[Theorem~2.8.1]{vershynin2018} applied to the \(b\) blocks gives
the stated bound.
\end{proof}

\subsection{The masked oracle}
\label{sec:noise-translation}
When the planted score is already close to the null score, the oracle answers
with the null score everywhere.  Otherwise, it reveals the planted
correction only on a high-likelihood region.  Write
\[
   G_\tau:=\frac{2R\sqrt d}{\tau^2},
\]
the pointwise score-gap bound from Lemma~\ref{lem:smoothing}.  Choose a
\(1\)-Lipschitz cutoff
\(\chi:\R\to[0,1]\) with \(\chi(u)=0\) for \(u\le0\) and
\(\chi(u)=1\) for \(u\ge1\), and define the \(L^p\) mask
\[
   \psi_{S,\tau}(x)
   :=\chi\left(r_{S,\tau}(x)+\log n-I_\tau+w/2+1\right).
\]

\begin{lemma}[Mask properties]\label{lem:mask-properties}
For sufficiently large \(A_{p,\rho,a}\), every \(\tau>0\) and
every \(n\ge1\), the following hold.
For every fixed codebook \(S\in V^n\),
\[
\begin{aligned}
   \pi_{S,\tau}(\psi_{S,\tau}<1)
   &\le2\exp\left[-c\min\left\{\frac{w^2}{d},w\right\}\right]
    \le\left(\frac{\eps\gamma}{2R\sqrt d}\right)^p,\\
   \Lip(\psi_{S,\tau})&\le G_\tau.
\end{aligned}
\]
For every fixed \(x\in\R^d\) and an i.i.d. codebook
\(S\sim\nu_{\mathrm{null}}^{\otimes n}\),
\[
   \Prob_S[\psi_{S,\tau}(x)>0]
   \le\exp(1+\log n-I_\tau+w/2).
\]
\end{lemma}

\begin{proof}
If \(X=Y^\star+\tau Z\sim\pi_{S,\tau}\), then the self term in
\eqref{eq:empirical-likelihood} shows that
\[
   \{\psi_{S,\tau}(X)<1\}
   \subseteq
   \{\imath_\tau(Y^\star,X)<I_\tau-w/2\}.
\]
By symmetry, the information-density tail is the same for every fixed
codeword.  Lemma~\ref{lem:information-concentration} gives the coverage
bound, using \(\log(2R\sqrt d/(\eps\gamma))\le C H\) for the last
inequality.  The Lipschitz bound follows from
\(\Lip(r_{S,\tau})\le G_\tau\).
Finally, \(\psi_{S,\tau}(x)>0\) requires
\(r_{S,\tau}(x)>I_\tau-\log n-w/2-1\), so
\eqref{eq:mean-empirical-likelihood} and Markov's inequality give the
blind-query bound.
\end{proof}

Let
\[
   \mathcal M_{p,\tau}(S)
   :=\E_{X\sim\pi_{S,\tau}}
      \|\nabla r_{S,\tau}(X)\|_2^p.
\]
The oracle is
\[
   \widehat s_{S,\tau}(x):=
   \begin{cases}
      s_{\mathrm{null},\tau}(x),
      & \mathcal M_{p,\tau}(S)\le(\eps/\tau)^p,\\[1mm]
      s_{\mathrm{null},\tau}(x)
      +\psi_{S,\tau}(x)\nabla r_{S,\tau}(x),
      & \mathcal M_{p,\tau}(S)>(\eps/\tau)^p.
   \end{cases}
\]
\begin{lemma}[The oracle is valid]\label{lem:oracle-valid}
For every fixed codebook \(S\) and every \(\tau\ge\gamma\),
\[
   \E_{X\sim\pi_{S,\tau}}
   \|\widehat s_{S,\tau}(X)-s_{S,\tau}(X)\|_2^p
   \le\frac{\eps^p}{\tau^p},
\]
and
\[
   \Lip(\widehat s_{S,\tau})
   \le\frac1{\tau^2}+\frac{5R^2d}{\tau^4}.
\]
\end{lemma}

\begin{proof}
The null branch has the required error moment by definition.  In the
masked branch, the error is
\(-(1-\psi_{S,\tau})\nabla r_{S,\tau}\), so
Lemma~\ref{lem:mask-properties} gives
\[
   \E_{\pi_{S,\tau}}
      \|\widehat s_{S,\tau}-s_{S,\tau}\|_2^p
   \le G_\tau^p\left(\frac{\eps\gamma}{2R\sqrt d}\right)^p
   =\frac{\eps^p\gamma^p}{\tau^{2p}}
   \le\frac{\eps^p}{\tau^p}.
\]
For regularity, write the masked oracle as
\((1-\psi_{S,\tau})s_{\mathrm{null},\tau}
+\psi_{S,\tau}s_{S,\tau}\).
Lemma~\ref{lem:smoothing} and \(\Lip(\psi_{S,\tau})\le G_\tau\)
give, by the product rule,
\[
   \Lip(\widehat s_{S,\tau})
   \le\frac1{\tau^2}+\frac{R^2d}{\tau^4}+G_\tau^2
   =\frac1{\tau^2}+\frac{5R^2d}{\tau^4}.
\]
The null branch satisfies the same bound by Lemma~\ref{lem:smoothing}.
\end{proof}

\paragraph{Admissibility in the theorem's notation.}

For a fixed codebook \(S\), set \(\pi_S:=\pi_{S,\gamma}\) and
\(\pi_{\mathrm{null}}:=\pi_{\mathrm{null},\gamma}\).
For \(\sigma>0\), put \(\tau=\sqrt{\gamma^2+\sigma^2}\) and define
\(\widehat s^S_\sigma:=\widehat s_{S,\tau}\).  Gaussian convolution gives
\[
   (\pi_S)_\sigma=\pi_{S,\tau},
   \qquad s_{\pi_S,\sigma}=s_{S,\tau},
   \qquad s_{\pi_{\mathrm{null}},\sigma}=s_{\mathrm{null},\tau}.
\]
Since \(\tau\ge\sigma\), Lemmas~\ref{lem:smoothing}
and~\ref{lem:oracle-valid} give the accuracy and Lipschitz bounds in
items~\textup{(2)}--\textup{(3)} of
Theorem~\ref{thm:main_parameterized}; item~\textup{(1)} follows from
\(\pi_S=\nu_S*\cN(0,\gamma^2I_d)\).
The map \(\sigma\mapsto\tau\) is injective, so it also preserves
the number of distinct query levels.

\section{The informative window is narrow}
\label{sec:window}

The informative window is centered at \(\log n=I_\tau\).
When \(\log n\) is below this scale, the likelihood mask hides the
signal from a fixed query.  When \(\log n\) is above it, averaging
over the codebook makes the planted score close to the null score.
We first prove the latter estimate in \(L^2\).

\begin{lemma}[Dense codebooks are score-close]\label{lem:dense-score}
There are universal constants \(c,C>0\) such that, for every
\(\tau>0\) and every integer \(n\ge1\) with
\(\Gamma:=\log n-I_\tau>0\),
\[
   \E_S\mathcal M_{2,\tau}(S)
   \le C G_\tau^2
      \exp\left[-c\min\left\{\frac{\Gamma^2}{d},\Gamma\right\}\right].
\]
\end{lemma}

\begin{proof}
Fix \(x\), and let \(m_{\mathrm{null},\tau}(x)\) be the null
posterior mean.  Write
\[
   W_i:=e^{\imath_\tau(Y_i,x)},\qquad
   u_i:=\frac{Y_i-m_{\mathrm{null},\tau}(x)}{2R\sqrt d},
   \qquad
   D:=\frac1n\sum_iW_i,\quad \bar u:=\frac1n\sum_iW_i u_i.
\]
Then \(\|u_i\|_2\le1\), \(\E W_i=1\), and \(\E W_i u_i=0\).
Tweedie's formula gives \(\nabla r_{S,\tau}(x)=G_\tau\bar u/D\).
Since \(f_{S,\tau}(x)=f_{\mathrm{null},\tau}(x)D\),
\begin{equation}\label{eq:score-ratio}
   \E_S\mathcal M_{2,\tau}(S)
   =G_\tau^2\int f_{\mathrm{null},\tau}(x)
          \E_S\frac{\|\bar u\|_2^2}{D}\,dx.
\end{equation}

Truncate the weights at \(T:=e^{I_\tau+\Gamma/2}\).  Let
\(D_\ell,\bar u_\ell\) be the same sums with only \(W_i\le T\) retained,
still normalized by \(n\), and let \(D_h:=D-D_\ell\).
Then
\[
   \frac{\|\bar u\|_2^2}{D}
   \le4\|\bar u_\ell\|_2^2+2D_h+2(1-D_\ell)^2.
\]
Indeed, if \(D\ge1/2\), use the squared triangle inequality and
\(\|\bar u-\bar u_\ell\|_2\le D_h\le D\).  If \(D<1/2\), the left side
is at most \(D<1/2\le2(1-D_\ell)^2\).

Put \(q(x):=\E[W_i\mathbf1_{\{W_i>T\}}]\le1\).
Independence and \(W_i^2\mathbf1_{\{W_i\le T\}}\le T W_i\) imply
\[
   \E\|\bar u_\ell\|_2^2\le T/n+q(x)^2,\qquad
   \E(1-D_\ell)^2\le T/n+q(x)^2,\qquad
   \E D_h=q(x).
\]
Here the bias of each truncated average is at most \(q(x)\),
since \(\E W_i u_i=0\) and \(\E W_i=1\).
Consequently, \(\E_S[\|\bar u\|_2^2/D]\le C(T/n+q(x))\).
Finally, \(T/n=e^{-\Gamma/2}\), and a change of measure gives
\[
   \int f_{\mathrm{null},\tau}(x)q(x)\,dx
   =\Prob[\imath_\tau>I_\tau+\Gamma/2].
\]
Substituting into \eqref{eq:score-ratio} and applying
Lemma~\ref{lem:information-concentration} proves the claim.
\end{proof}

\begin{proposition}[A narrow informative window]\label{prop:window}
For a sufficiently large \(A_{p,\rho,a}\), every \(\tau\ge\gamma\) and
every integer \(n\ge1\) satisfy the following for an i.i.d. codebook
\(S\) of size \(n\).

\begin{enumerate}[label=\textup{(\alph*)},leftmargin=2.2em]
   \item \emph{Sparse side.}  If
   \[
      \log n\le I_\tau-w,
   \]
   then, for every fixed \(x\in\R^d\),
   \[
      \Prob_S\left[
         \widehat s_{S,\tau}(x)\ne s_{\mathrm{null},\tau}(x)
      \right]\le\delta_d.
   \]

   \item \emph{Dense side.}  If
   \[
      \log n\ge I_\tau+w,
   \]
   then
   \[
      \Prob_S\left[
         \widehat s_{S,\tau}(\,\cdot\,)
         \equiv s_{\mathrm{null},\tau}(\,\cdot\,)
      \right]\ge1-\delta_d.
   \]
\end{enumerate}
\end{proposition}

Part~\textup{(a)} does not assert simultaneous agreement at all query
points: its fixed-query bound suffices for the first-divergence coupling
below.  Part~\textup{(b)} instead puts the entire oracle in its null branch.

\begin{proof}
On the sparse side, a non-null answer requires
\(\psi_{S,\tau}(x)>0\).  By Lemma~\ref{lem:mask-properties},
\[
   \Prob_S[\widehat s_{S,\tau}(x)\ne s_{\mathrm{null},\tau}(x)]
   \le e^{1+\log n-I_\tau+w/2}
   \le e^{1-w/2}\le\delta_d.
\]

On the dense side, \(\mathcal M_{p,\tau}(S)\le
G_\tau^{p-2}\mathcal M_{2,\tau}(S)\), since \(p\ge2\).
Lemma~\ref{lem:dense-score} and Markov's inequality give
\[
   \Prob_S[\mathcal M_{p,\tau}(S)>(\eps/\tau)^p]
   \le C\left(\frac{\tau G_\tau}{\eps}\right)^p
      \exp\left[-c\min\left\{\frac{\Gamma^2}{d},\Gamma\right\}\right],
   \qquad \Gamma:=\log n-I_\tau.
\]
Since \(\tau\ge\gamma\), the logarithm of the prefactor, including
\(1/\delta_d\), is at most \(C_{p,\rho,a}H\).  For \(\Gamma\ge w\),
this is absorbed by the exponent when \(A_{p,\rho,a}\) is sufficiently
large, giving failure probability at most \(\delta_d\).
\end{proof}

\section{From oracle indistinguishability to sampling lower bounds}
\label{sec:completion}

At any fixed noise level, Proposition~\ref{prop:window} allows only a
short interval of log-codebook sizes to affect the oracle.  We
therefore draw the size from a log-uniform prior: every short interval
has small prior mass.  Comparing the planted execution with the common
null run handles adaptivity; a likelihood witness then gives the
sampling lower bound.

\subsection{Coupling to the common null execution}
\label{sec:blind-output}

Set
\[
   m_d:=\frac b{16}\log M,
   \qquad n_{\max}:=\lfloor e^{m_d}\rfloor.
\]
The factor \(1/16\) leaves \(\log n+m_d\) a linear margin below
\(I_\gamma\) for \(n\le n_{\max}\), using
\(I_\gamma\ge(b/5)\log M\), proved in Lemma~\ref{lem:tv-exposure}.

Draw
\[
   U\sim\Unif[0,m_d],
   \qquad
   N:=\lfloor e^U\rfloor,
\]
and, conditional on \(N=n\), draw
\[
   S=(Y_1,\ldots,Y_n),
   \qquad Y_i\stackrel{\mathrm{iid}}\sim\nu_{\mathrm{null}}.
\]
For an algorithm \(\mathcal A\), abbreviate its planted output by
\[
   \widehat X_S^{\mathcal A}
   :=\widehat X_{\pi_S,\widehat s^S}^{\mathcal A}.
\]
Let \(\mathcal D\) be the induced law of
\((\pi_S,\{\widehat s^S_\sigma\}_{\sigma>0})\).
This law is independent of the algorithm.

\begin{proposition}[Level-sensitive null coupling]
\label{prop:blind-output}
Let \(\mathcal A\) be any randomized adaptive algorithm making at most
\(Q\) queries and using at most \(\ell\) distinct noise levels on every execution
path.  Couple its planted and null executions using the same internal
randomness \(\Xi\), independent of \(S\) and including its output
randomness, and let \(\mathsf{Div}\) be the event that the two executions
differ.  Then
\[
   \Prob_{S,\Xi}(\mathsf{Div})
   \le\frac{C\ell w}{m_d}+Q\delta_d.
\]
\end{proposition}

\begin{proof}
Fix \(\Xi=\xi\) and run \(\mathcal A\) against
\((\pi_{\mathrm{null}},\{s_{\pi_{\mathrm{null}},\sigma}\}_{\sigma>0})\).
Convert its query levels to \(\tau=\sqrt{\gamma^2+\sigma^2}\)
as in Section~\ref{sec:noise-translation}.  This run contains at most
\(Q\) queries and at most \(\ell\) distinct levels.  Since
\(0\le U-\log N\le\log2\), the event
\(|\log N-I_\tau|\le w\) confines \(U\) to an interval of length
\(O(w)\).  Thus the probability that any null-run level is
informative is at most \(C\ell w/m_d\).

Outside this event, couple the planted and null executions using the
same \(\xi\).  If they diverge, then immediately before their first
differing answer they have the same transcript and ask one of the
fixed null-run queries.  That query is independent of \(S\) and lies
outside the informative window, so Proposition~\ref{prop:window}
bounds its disagreement probability by \(\delta_d\).  A union bound
over the at most \(Q\) possible first differing queries, followed by
averaging over \((\Xi,N)\), proves the bound.
\end{proof}

\subsection{A target-only likelihood witness}
\label{sec:tv-separation}

\begin{lemma}[A clipped likelihood witness]\label{lem:tv-exposure}
For all sufficiently large \(d\), depending only on \(\rho\), and
every codebook \(S\in V^n\) with \(1\le n\le n_{\max}\), the function
\[
   g_S(x):=\min\left\{1,\frac{(r_{S,\gamma}(x))_+}{m_d}\right\}
\]
satisfies
\[
   \E_{\pi_S}g_S\ge1-\frac\rho2.
\]
Moreover, for every fixed \(x\in\R^d\) and
\(S\sim\nu_{\mathrm{null}}^{\otimes n}\),
\[
   \E_S g_S(x)\le\frac1{m_d},
\]
uniformly over \(1\le n\le n_{\max}\).
\end{lemma}

\begin{proof}
Nearest-neighbor decoding in one planar block has error probability at most
\[
   2\overline\Phi\left(
      \frac R\gamma\sin(\pi/M)
   \right)
   \le2\overline\Phi(4/5)<7/16,
\]
where \(\overline\Phi(t):=\Prob_{G\sim\cN(0,1)}[G\ge t]\).
Fano's inequality \cite[Theorem~2.10.1 and the following
corollary]{coverthomas2006} therefore gives
\[
   I_\gamma
   \ge b\left(\frac9{16}\log M-\log2\right)
   \ge\frac b5\log M,
\]
where the last inequality uses \(M\ge7\).

If \(X=Y^\star+\gamma Z\sim\pi_S\), the self term gives
\(r_{S,\gamma}(X)\ge\imath_\gamma(Y^\star,X)-\log n\).
The choice of \(n_{\max}\) leaves the margin
\[
   I_\gamma-\log n-m_d
   \ge\left(\frac15-\frac18\right)b\log M
   =\frac{3b}{40}\log M.
\]
By symmetry and Lemma~\ref{lem:information-concentration},
\(\pi_S(r_{S,\gamma}<m_d)\le2e^{-cd}\le\rho/2\) for large \(d\),
for every fixed \(S\).  This proves the target expectation bound.

For a fixed \(x\), use \(t_+\le e^t\) and
\eqref{eq:mean-empirical-likelihood}:
\[
   \E_S g_S(x)
   \le\frac{\E_S(r_{S,\gamma}(x))_+}{m_d}
   \le\frac{\E_S e^{r_{S,\gamma}(x)}}{m_d}
   =\frac1{m_d}.
\]
\end{proof}

\begin{proof}[Proof of Theorem~\ref{thm:main_parameterized}]
Section~\ref{sec:noise-translation} verifies
items~\textup{(1)}--\textup{(3)} for every realization of the hard law
constructed in Section~\ref{sec:blind-output}.
Since \(m_d\asymp d\log(R/\gamma)\), we have
\(B_dw/m_d=O_{p,\rho,a}(1)\).  Choose the theorem's resource constant
\(c=c(p,\rho,a)\le1\) sufficiently small.  Then any algorithm
\(\mathcal A\) allowed by item~\textup{(4)} has
\(\Prob_{S,\Xi}(\mathsf{Div})\le\rho^2/4\) by
Proposition~\ref{prop:blind-output}.
Let \(\widehat X_0^{\mathcal A}\) be the null output in that coupling.
It is independent of \(S\), and equals \(\widehat X_S^{\mathcal A}\)
outside \(\mathsf{Div}\).
Lemma~\ref{lem:tv-exposure} and Fubini therefore give, for \(d\)
sufficiently large,
\[
\begin{aligned}
   \E_S\!\left[
      \E\bigl[g_S(\widehat X_S^{\mathcal A})\mid S\bigr]
   \right]
   &\le
   \Prob_{S,\Xi}(\mathsf{Div})
   +\E_{S,\Xi}g_S(\widehat X_0^{\mathcal A})\\
   &\le\frac{\rho^2}{4}+\frac1{m_d}
    \le\frac{\rho^2}{2}.
\end{aligned}
\]
Thus, by Markov's inequality, with probability at least \(1-\rho\)
over \(S\),
\[
   \E\bigl[g_S(\widehat X_S^{\mathcal A})\mid S\bigr]
   \le\rho/2.
\]
For every such \(S\), the variational characterization of total
variation yields
\[
   \TV(\widehat X_S^{\mathcal A},\pi_S)
   \ge\E_{\pi_S}g_S
      -\E\bigl[g_S(\widehat X_S^{\mathcal A})\mid S\bigr]
   \ge1-\rho.
\]
This proves item~\textup{(4)}.  The ``in particular''
claim follows because a total budget \(cB_d\) gives both
\(Q\le cB_d\le d^a\) and \(\ell\le Q\le cB_d\).
All dimension thresholds used above depend only on \(p,\rho,a,R/\gamma\).
\end{proof}

\paragraph{Derivation of Theorem~\ref{thm:main_simplified}.}
Take \(a=1\), \(\rho=0.01\), fixed constants \(R>2\gamma>0\), and
\(\varepsilon_{\rm err}=d^{-O(1)}\) in
Theorem~\ref{thm:main_parameterized}.  Then
\(B_d=\Omega(\sqrt{d/\log d})\), and the ``in particular'' clause
gives Theorem~\ref{thm:main_simplified}.

\paragraph{Testing consequence.}
Give equal prior probability to the null transcript and to the
transcript from a random planted instance.  Proposition~\ref{prop:blind-output}
bounds the optimal testing advantage by
\(\tfrac12[C\ell w/m_d+Q\delta_d]\).  Thus advantage at
least \(\alpha\), together with \(Q\delta_d\le\alpha\), requires
\[
   \ell=\Omega(\alpha m_d/w)
   =\Omega_{p,\rho,a}(\alpha B_d)
\]
distinct levels.

\section{An extension to non-uniform weights}
\label{sec:nonuniform-weights}

\begin{proposition}[Robustness to polynomially bounded reweighting]
\label{prop:bounded-reweighting}
Fix a constant \(C_{\mathrm{wt}}\ge0\).  For each
\(1\le n\le n_{\max}\), fix a deterministic probability vector
\(\eta^{(n)}=(\eta^{(n)}_1,\ldots,\eta^{(n)}_n)\) satisfying
\[
   n\|\eta^{(n)}\|_\infty\le d^{C_{\mathrm{wt}}}.
\]
Replace \(\nu_S\) in the hard construction by
\(\nu^\eta_S:=\sum_{i=1}^n\eta^{(n)}_i\delta_{Y_i}\), and let
\(\mathcal D^\eta\) be the induced weighted hard law.  Then
\(\mathcal D^\eta\) witnesses Theorem~\ref{thm:main_parameterized}
with the same \(B_d\); the constants \(c\) and \(d_0\) may additionally
depend on \(C_{\mathrm{wt}}\).  The result also holds for a random
weight family drawn independently of \(U\) and all codeword locations
and satisfying the
displayed bound almost surely.
\end{proposition}

\begin{proof}
As before, it suffices to consider \(p\ge2\).
Fix \(n\), abbreviate \(\eta:=\eta^{(n)}\), and put
\(\eta_{\max}:=\max_i\eta_i\).
Use the weighted mixture in the likelihood ratio and oracle, enlarging
the constant in \(w\) to depend also on \(C_{\mathrm{wt}}\).
For every fixed \(x\),
\[
   e^{r^\eta_{S,\tau}(x)}
   =\sum_i\eta_i e^{\imath_\tau(Y_i,x)},
   \qquad \E_S e^{r^\eta_{S,\tau}(x)}=1.
\]
The sparse-side bound follows by Markov's inequality.

For mask coverage, generate a target sample as \(X=Y_J+\tau Z\),
where \(J\sim\eta\).  The self term gives
\(r^\eta_{S,\tau}(X)\ge\imath_\tau(Y_J,X)+\log\eta_J\), and
\(\Prob[\log(n\eta_J)<-t]\le e^{-t}\): the small weights have
total mass at most \(n(e^{-t}/n)\).
Combining this with information concentration, taking both deviations
to be \(h/2\), gives, for every fixed \(S\) and \(h\ge0\),
\[
   \pi^\eta_{S,\tau}
   \bigl(r^\eta_{S,\tau}+\log n<I_\tau-h\bigr)
   \le C e^{-c\min\{h^2/d,h\}}.
\]
This bound gives mask coverage and the target witness estimate
\(\E_{\pi^\eta_S}g^\eta_S\ge1-\rho/2\), with the same \(m_d\)
and \(g^\eta_S:=\min\{1,(r^\eta_{S,\gamma})_+/m_d\}\).
The mean-one likelihood identity gives \(\E_Sg^\eta_S(x)\le1/m_d\).

On the dense side, the variance bound \(T/n\) in the proof of
Lemma~\ref{lem:dense-score} becomes
\[
   T\sum_i\eta_i^2\le T\eta_{\max}.
\]
The same truncation argument therefore applies with information gap
\[
   \log(1/\eta_{\max})-I_\tau
   \ge\log n-I_\tau-C_{\mathrm{wt}}\log d.
\]
Since \(H\ge\log d\), enlarging the constant in \(w\) absorbs this
loss and gives the dense-side bound with failure probability
\(\delta_d\).  Proposition~\ref{prop:blind-output} and the witness
argument now prove the theorem with a smaller resource constant \(c\).
All bounds are uniform over the allowed weight families, so conditioning
on an independent random weight family and then averaging proves the
last assertion.
\end{proof}

\section{Subexponential score errors}
\label{sec:subexponential}

Only the oracle construction and the resulting window scale change
from the \(L^p\) proof.

\begin{theorem}[Subexponential-error extension]
\label{thm:subexponential}
Theorem~\ref{thm:main_parameterized} remains valid after removing the parameter \(p\) and
making the following substitutions.  Denote the oracle by
\(\widehat s^{\mathrm{se}}\), and replace item~\textup{(2)} by
\[
   \Prob_{X\sim\pi_\sigma}\left[
      \|\widehat s^{\mathrm{se}}_\sigma(X)
        -s_{\pi,\sigma}(X)\|_2\ge z
   \right]
   \le 2\exp\left(-\frac{z\sigma}{\eps}\right)
   \qquad(\sigma>0,\ z\ge0).
\]
In the resource bound, replace \(H\) and \(B_d\) by the quantities
\[
   H_{\mathrm{se}}
   :=\frac{R\sqrt d}{\gamma\eps},
   \qquad
   B_d^{\mathrm{se}}
   :=\frac{d\log(R/\gamma)}
   {\sqrt{dH_{\mathrm{se}}}+H_{\mathrm{se}}}.
\]
The constants may be chosen as \(c=c(\rho,a)>0\) and
\(d_0=d_0(\rho,a,R/\gamma)\).
All other conclusions of the theorem are unchanged.
\end{theorem}

\begin{proof}
Use the target family and size prior from Section~\ref{sec:blind-output}.
In the mask and oracle construction, replace \(w\) by
\[
   w_{\mathrm{se}}:=A_{\rho,a}^{\mathrm{se}}
      \bigl(\sqrt{dH_{\mathrm{se}}}+H_{\mathrm{se}}\bigr),
\]
and use the globally null branch exactly when
\[
   \E_{\pi_{S,\tau}}
   e^{\tau\|\nabla r_{S,\tau}\|_2/\eps}\le2.
\]
Write \(\psi^{\mathrm{se}}_{S,\tau}\) and
\(\widehat s^{\mathrm{se}}_{S,\tau}\) for this mask and oracle.
The Lipschitz bound follows as in Lemma~\ref{lem:oracle-valid}.

Put
\[
   u_\tau:=\frac{\tau G_\tau}{\eps}
          =\frac{2R\sqrt d}{\tau\eps}.
\]
At total noise \(\tau\ge\gamma\), the error's exponential moment
is at most \(2\) in the null branch by definition.
For the masked branch, the coverage argument in
Lemma~\ref{lem:mask-properties} gives
\(\pi_{S,\tau}(\psi^{\mathrm{se}}_{S,\tau}<1)
\le e^{-u_\gamma}\) for sufficiently large
\(A_{\rho,a}^{\mathrm{se}}\), since \(u_\gamma=2H_{\mathrm{se}}\).
The error is at most \(G_\tau\) and vanishes wherever the mask is one, so
\[
   \E_{\pi_{S,\tau}}
      e^{\tau\|\widehat s^{\mathrm{se}}_{S,\tau}-s_{S,\tau}\|_2/\eps}
   \le1+e^{u_\tau}\pi_{S,\tau}(\psi^{\mathrm{se}}_{S,\tau}<1)
   \le2,
\]
since \(u_\tau\le u_\gamma\).
Exponential Markov gives the required tail in either branch.

The sparse-window proof also applies with \(w_{\mathrm{se}}\).
For the dense side, \(e^v-1\le ve^{u_\tau}\) for
\(0\le v\le u_\tau\), Markov's inequality, and Cauchy--Schwarz give
\[
\begin{aligned}
   \Prob_S\!\left[
      \E_{\pi_{S,\tau}}e^{\tau\|\nabla r_{S,\tau}\|_2/\eps}>2
   \right]
   &\le\E_S\E_{\pi_{S,\tau}}
      \left[e^{\tau\|\nabla r_{S,\tau}\|_2/\eps}-1\right]\\
   &\le\frac{\tau e^{u_\tau}}{\eps}
      \sqrt{\E_S\mathcal M_{2,\tau}(S)}\\
   &\le C\exp\left(C H_{\mathrm{se}}
      -c\min\{\Gamma^2/d,\Gamma\}\right),
\end{aligned}
\]
where \(\Gamma:=\log n-I_\tau>0\); the last line uses
Lemma~\ref{lem:dense-score} and \(u_\tau\le2H_{\mathrm{se}}\).
Since \(\log(1/\delta_d)=O_{\rho,a}(H_{\mathrm{se}})\),
\(\Gamma\ge w_{\mathrm{se}}\) makes this probability at most
\(\delta_d\) after enlarging \(A_{\rho,a}^{\mathrm{se}}\).

Proposition~\ref{prop:blind-output} and Lemma~\ref{lem:tv-exposure}
give the level bound, since
\(m_d/w_{\mathrm{se}}\asymp_{\rho,a}B_d^{\mathrm{se}}\).
Translate the oracle to \(\sigma\) as in
Section~\ref{sec:noise-translation}; the tail bound follows from
\(\sqrt{\gamma^2+\sigma^2}\ge\sigma\).
Since \(B_d^{\mathrm{se}}\le d\le d^a\), the total-query conclusion
follows as well.
\end{proof}

For fixed \(R/\gamma\) and \(\eps\), we have
\(H_{\mathrm{se}}\asymp\sqrt d\) and
\(B_d^{\mathrm{se}}\asymp d^{1/4}\).
Applying the theorem with \(a=1\), \(\rho=0.01\), fixed
\(R>2\gamma>0\), and accuracy \(\min\{\varepsilon_{\mathrm{err}},1\}\)
proves Theorem~\ref{thm:psi1_simplified} for every constant
\(\varepsilon_{\mathrm{err}}>0\).

\section{A Wasserstein extension}
\label{sec:wasserstein}

Rescaling the clipped likelihood witness in Lemma~\ref{lem:tv-exposure}
gives a \(1\)-Lipschitz test function.

\begin{corollary}[Wasserstein lower bound]\label{cor:wasserstein}
In the setting of Theorem~\ref{thm:main_parameterized}, for every
\(q\in[1,\infty)\), the same hard law \(\mathcal D\) constructed in
Section~\ref{sec:blind-output} satisfies
\[
   \Prob_{(\pi,\widehat s)\sim\mathcal D}
   \left[
      W_q(\widehat X_{\pi,\widehat s}^{\mathcal A},\pi)
      \ge c_W\frac{\gamma^2\log(R/\gamma)}{R}\sqrt d
   \right]
   \ge1-\rho
\]
for every algorithm obeying the query budgets in item~\textup{(4)},
where \(c_W>0\) is universal.  The total-query conclusion also holds.
\end{corollary}

\begin{proof}
Use the same resource constant and high-probability witness bound as
in the proof of Theorem~\ref{thm:main_parameterized}.
Since \(\Lip(r_{S,\gamma})\le G_\gamma\),
\(\phi_S:=(m_d/G_\gamma)g_S\) is bounded and \(1\)-Lipschitz.
For each instance in that event, the coupling definition of \(W_1\) gives
\[
   W_1(\widehat X_S^{\mathcal A},\pi_S)
   \ge \E_{\pi_S}\phi_S
      -\E[\phi_S(\widehat X_S^{\mathcal A})\mid S]
   \ge(1-\rho)\frac{m_d}{G_\gamma}.
\]
Since \(b\ge d/3\), \(\log M\ge\log(R/\gamma)\), and \(\rho<1/4\),
\[
   (1-\rho)\frac{m_d}{G_\gamma}
   =(1-\rho)\frac{b\gamma^2\log M}{32R\sqrt d}
   \ge\frac{\gamma^2\log(R/\gamma)}{128R}\sqrt d .
\]
Finally, \(W_q\ge W_1\) for \(q\ge1\), with the extended-value
interpretation when moments are infinite.
\end{proof}

\paragraph{Derivation of Theorem~\ref{thm:wasserstein_simplified}.}
Take \(a=1\), \(\rho=0.01\), fixed constants \(R>2\gamma>0\), and
\(\varepsilon_{\rm err}=d^{-O(1)}\) in Corollary~\ref{cor:wasserstein}.
Then \(B_d=\Omega(\sqrt{d/\log d})\), and the total-query conclusion gives
Theorem~\ref{thm:wasserstein_simplified} with
\(c=c_W\gamma^2\log(R/\gamma)/R\).

\section{A synthetic illustration of the informative window}
\label{app:synthetic}

\paragraph{Model and proxy.}
Fix \(\kappa\in(0,\log2)\), include each vertex of \(\{-1,+1\}^d\)
in \(S\) independently with probability \(p_d=e^{\kappa d}/2^d\),
and condition on \(S\ne\varnothing\).  If \(\mathcal H_k\) is the
radius-\(k\) Hamming shell about \(\mathbf1\), write
\(C_k:=\binom dk\) and \(N_k:=|S\cap\mathcal H_k|\).
Before conditioning, the \(N_k\)'s are independent
\(\operatorname{Bin}(C_k,p_d)\) variables; given their values, the
selected vertices are uniform within each shell.

At total noise \(\tau>0\), set
\[
\begin{gathered}
   \omega_\tau:=e^{-2/\tau^2},\qquad
   \mathcal Z_\tau:=\sum_kN_k\omega_\tau^k,\qquad
   \Lambda_\tau:=
   \frac{2^d\mathcal Z_\tau}{|S|(1+\omega_\tau)^d},\\
   \bar m_\tau:=
   \frac{\sum_kN_k\omega_\tau^k(1-2k/d)}{\mathcal Z_\tau},
   \qquad
   m_{\mathrm{null},\tau}:=\frac{1-\omega_\tau}{1+\omega_\tau}.
\end{gathered}
\]
For the \(\tau\)-Gaussian-smoothed uniform laws on \(S\) and the full
hypercube, let \(s^{\mathrm{hc}}_{S,\tau}\) and
\(s^{\mathrm{hc}}_{\mathrm{null},\tau}\) denote their scores;
\(\Lambda_\tau\) is the ratio of their densities at \(\mathbf1\).
Tweedie's formula gives
\[
\begin{aligned}
   \widetilde{\mathcal I}_\tau
   &:=\E\!\left[
      \frac{\Lambda_\tau}{d}
      \|s^{\mathrm{hc}}_{S,\tau}(\mathbf 1)
        -s^{\mathrm{hc}}_{\mathrm{null},\tau}(\mathbf 1)\|_2^2
      \,\middle|\,(N_k)_{k=0}^d\right]\\
   &=\tau^{-4}\Lambda_\tau\left[
      (\bar m_\tau-m_{\mathrm{null},\tau})^2
      +\frac1{\mathcal Z_\tau^2}
       \sum_{k=1}^{d-1}N_k\omega_\tau^{2k}
       \frac{4k(d-k)}{d^2}
       \frac{C_k-N_k}{C_k-1}
   \right].
\end{aligned}
\]
The two terms in brackets are the squared shift of the conditional
posterior mean and the within-shell variance; sampling without
replacement gives the factor \((C_k-N_k)/(C_k-1)\).
\paragraph{Connection with the informative window.}
Define
\[
   q_\tau:=\frac{\omega_\tau}{1+\omega_\tau}
      =\frac1{1+e^{2/\tau^2}}.
\]
The null weights concentrate at \(k=q_\tau d+O(\sqrt d)\).
For fixed \(q\in(0,1)\), the unconditioned shell occupancy satisfies
\[
   \log \E N_{\lfloor qd\rfloor}
      =d\bigl(h(q)-\log2+\kappa\bigr)+O(\log d),
\]
where \(h(q):=-q\log q-(1-q)\log(1-q)\).  Let
\(q_\kappa\in(0,1/2)\) solve \(h(q_\kappa)=\log2-\kappa\).
The occupancy changes from sparse to dense at \(q=q_\kappa\), so shell
fluctuations localize the signal to
\(q_\tau-q_\kappa=O(d^{-1/2})\).  Since
\(\mathrm d q_\tau/\mathrm d(\log\tau)\ne0\) at the matching scale,
the predicted width in \(\log\tau\) is \(O(d^{-1/2})\).
The figure fixes \(\kappa\) and varies \(\tau\), whereas
Proposition~\ref{prop:window} fixes \(\tau\) and varies \(\log n\).

Figure~\ref{fig:main-synthetic} uses \(\kappa=0.2\),
\(d=2^9,\ldots,2^{15}\), and the Poisson approximation
\(N_k\sim\operatorname{Pois}(C_kp_d)\) to the independent shell counts.
It plots the
pointwise median and its FWHM in \(\log\tau\); the latter is approximately
linear in \(d^{-1/2}\), consistent with the prediction above.

\end{document}